\PassOptionsToPackage{usenames,dvipsnames}{xcolor} 

\documentclass[sigconf]{aamas} 

\usepackage{balance} 

\setcopyright{ifaamas}
\acmConference[AAMAS '21]{Proc.\@ of the 20th International Conference on Autonomous Agents and Multiagent Systems (AAMAS 2021)}{May 3--7, 2021}{Online}{U.~Endriss, A.~Now\'{e}, F.~Dignum, A.~Lomuscio (eds.)}
\copyrightyear{2021}
\acmYear{2021}
\acmDOI{}
\acmPrice{}
\acmISBN{}

\acmSubmissionID{426}

\title{Egalitarian Judgment Aggregation}

\author{Sirin Botan}
\affiliation{
	\institution{University of Amsterdam, The Netherlands}}
\email{sirin.botan@uva.nl}

\author{Ronald de Haan}
\affiliation{
	\institution{University of Amsterdam, The Netherlands}}
\email{me@ronalddehaan.eu}

\author{Marija Slavkovik}
\affiliation{
	\institution{University of Bergen, Norway}
}
\email{marija.slavkovik@uib.no}

\author{Zoi Terzopoulou}
\affiliation{
	\institution{University of Amsterdam, The Netherlands}}
\email{z.terzopoulou@uva.nl}

\begin{abstract}
	Egalitarian considerations play a central role  in many areas of social choice theory. Applications of egalitarian principles range from ensuring everyone gets an equal share of a cake when deciding how to divide it, to guaranteeing balance with respect to gender or ethnicity in committee elections. Yet, the egalitarian approach has received little attention  in judgment aggregation---a powerful framework for aggregating logically interconnected issues. We make the first steps towards filling that gap. We introduce axioms capturing two classical interpretations of egalitarianism in judgment aggregation and situate these within the context of existing axioms in the pertinent framework of belief merging. We then explore the relationship between these axioms and several notions of strategyproofness from social choice theory at large. Finally, a novel egalitarian judgment aggregation rule stems from our analysis; we present complexity results concerning both outcome determination and strategic manipulation for that rule.
\end{abstract}


\usepackage{booktabs}    
\usepackage{enumitem} 
\usepackage{balance}
\usepackage{amsmath}
\usepackage{amsthm}
\usepackage[multiple]{footmisc}
\usepackage{newtxmath}
\usepackage[shortcuts]{extdash} 
\usepackage{tikz}
\usepackage{tabularx}
\usepackage{bm}
\usepackage{listings}
\usepackage{lstautogobble}
\usepackage{xcolor,colortbl}
\usetikzlibrary{arrows.meta}
\renewcommand{\phi}{\varphi}
\DeclareMathOperator*{\argmin}{argmin}

\usepackage{pifont}

\usepackage[colorinlistoftodos]{todonotes}

\newcommand{\N}{N}

\newcommand{\Pf}[0]{\ensuremath{\bm{J}}}

\newcommand{\pess}{\succ^{\textit{pess}}}

\newcommand{\opt}{\succ^{\textit{opt}}}

\newcommand{\setsucc}{\mathrel{\mathring{\succ}}}
\newcommand{\setsucceq}{\mathrel{\mathring{\succeq}}}

\newcommand{\NP}{\normalfont\textsf{NP}}

\newcommand{\SigmaP}[1]{\ensuremath{\Sigma^{\text{\normalfont\textsf{p}}}_{#1}}}

\newcommand{\ThetaP}[1]{\ensuremath{\Theta^{\text{\normalfont\textsf{p}}}_{#1}}}

\newcommand{\FNP}{\normalfont\textsf{FNP}}
\newcommand{\PNPlog}[0]{\ensuremath{\normalfont\textsf{P}^{\textsf{NP}}\textsf{[log]}}}
\newcommand{\PNPpar}[0]{\ensuremath{\normalfont\textsf{P}^{\textsf{NP}}_{\smash{\scalebox{0.7}{||}}}}}
\newcommand{\PNPlogwit}[0]{\ensuremath{\normalfont\textsf{P}^{\textsf{NP}}\textsf{[log,wit]}}}
\newcommand{\FPNPlogwit}[0]{\ensuremath{\normalfont\textsf{FP}^{\textsf{NP}}\textsf{[log,wit]}}}
\newcommand{\FNPOptPlog}[0]{\ensuremath{\normalfont\textsf{FNP}/\!\!/\textsf{OptP[log]}}}

\theoremstyle{defn}
\newtheorem{defn}{Definition}
\newtheorem{ex}{Example}
\theoremstyle{plain}

\newtheorem{thm}{Theorem}

\newtheorem{prop}{Proposition}
\newtheorem{cor}{Corollary}
\newtheorem{lem}{Lemma}
\newtheorem{cla}{Claim}
\newtheorem{obs}[cla]{Observation}

\newenvironment{claimproof}[1][Proof]{\begin{proof}[#1]}{\end{proof}}

\newcommand{\duplicatethm}[2]{{
		\renewcommand{\thethm}{#1}
		\begin{thm}
			#2
		\end{thm}
		\addtocounter{thm}{-1}
}}
\newcommand{\duplicateprop}[2]{{
		\renewcommand{\theprop}{#1}
		\begin{prop}
			#2
		\end{prop}
		\addtocounter{prop}{-1}
}}

\def\ContinueLineNumber{\lstset{firstnumber=last}}
\def\StartLineAt#1{\lstset{firstnumber=#1}}

\newcommand{\inlinecode}[1]{\texttt{\small #1}}

\DeclareRobustCommand{\DE}[3]{#2}

\begin{document}
	
	
	\pagestyle{fancy}
	\fancyhead{}
	
	
	\maketitle 
	

	\section{Introduction}
	Judgment aggregation is an area of social choice theory concerned with turning the
	individual binary judgments of a group of agents over logically related issues into a collective judgment~\citep{endriss2016}. 
	Being a flexible and widely applicable framework, judgment aggregation provides the foundations for collective decision making settings in various disciplines, like philosophy, economics, legal theory, and artificial intelligence \citep{grossi2014judgment}. The purpose of judgment aggregation methods
	(\emph{rules}) is to find those collective judgments that better represent the group as a whole.  Following the utilitarian approach in social choice, an ``ideal" such collective judgment has traditionally been considered the will of the majority. In this paper we challenge this perspective, introducing a more egalitarian point of view. 
	
	
	
	In economic theory, utilitarian approaches are often contrasted with egalitarian ones~\citep{Moulin}.  In the context of judgment aggregation,  an egalitarian rule must take into account whether the collective outcome achieves equally distributed satisfaction among agents and  ensure that agents enjoy equal consideration. A rapidly growing application domain of egalitarian judgment aggregation (that also concerns multiagent systems with practical implications like in the construction of self-driving cars) is the aggregation of moral choices~\citep{Conitzer:2017}, where utilitarian approaches do not always offer appropriate solutions~\citep{noothigattu2018,Baum2017}. One of the  drawbacks of majoritarianism is that a strong enough majority can cancel out the views of a minority, which is  questionable in several occasions.
	
	For example, suppose that the president of a student union has secured some budget for the decoration of the union's office and she asks her colleagues for their opinions on which paintings to buy (perhaps imposing some constraints on the combinations of paintings that can be simultaneously selected, due to clashes on style). If the members of the union largely consist of pop-art enthusiasts that the president tries to satisfy, then a few members with diverting taste will find themselves in an  office that they detest; an arguably more viable strategy would be to ensure that---as much as possible---no-one is strongly dissatisfied. But then, consider a similar situation in which a kindergarten teacher needs to decide what toys to complement the existing playground with. In that case, the teacher's goal is to select toys that equally (dis)satisfy all kids involved, so that no extra tension is created due to envy, which the teacher will have to resolve---if the kids disagree a lot, then the teacher may end up choosing toys that none of them really likes. 
	
	In order to formally capture scenarios like the above, this paper introduces two fundamental properties (also known as \emph{axioms}) of egalitarianism to judgment aggregation, inspired by the theory of justice. The first captures the idea behind the so-called \emph{veil of ignorance} of \citet{rawls1971theory}, while the second speaks about how happy agents are with the collective outcome relative to each other. 
	
	Our axioms closely mirror properties in other areas of social choice theory. In \emph{belief merging}, egalitarian axioms and merging operators have been studied by \citet{everaere2014egalitarian}. The nature of their axioms is in line with the interpretation of egalitarianism in this paper, although the two main properties they study are logically weaker than ours, as we further discuss in Section~\ref{sec:bel-merg}.
	In  \emph{resource allocation}, fairness has  been interpreted both as maximising the share of the worst off agent \citep{budish2011combinatorial} as well as eliminating envy between agents \citep{foley1967resource}. In  \emph{multiwinner elections}, egalitarianism is present in diversity \citep{elkind2017properties} and in proportional representation \citep{dummet84,aziz2017justified} notions.

	Unfortunately, egalitarian considerations often come at a cost. A central concern in many areas of social choice theory, of which judgement aggregation does not constitute an exception, is that agents may have incentives to \emph{manipulate}, i.e., to misrepresent their judgments aiming for a more preferred outcome \citep{dietrich2007strategy}.  Frequently, it is impossible to simultaneously be fair and avoid strategic manipulation. For both variants of fairness in resource allocation, rules satisfying them usually are susceptible to strategic manipulation \citep{chen2013truth,amanatidis2016,brams2008proportional,mossel2010truthful}. The same type of results have recently been obtained for multiwinner elections \citep{lackner2018approval,peters2018proportionality}. It is not easy to be egalitarian while disincentivising agents from taking advantage of it.
	
	Inspired by notions of manipulation stemming from voting theory, we explore how our egalitarian axioms affect the agents' strategic behaviour within  judgment aggregation. Our most important result in this vein is showing that the two properties of egalitarianism defined in this paper clearly differ in terms of strategyproofness. 
	
	Our axioms give rise to two concrete egalitarian  rules---one that has been previously studied, and one that is new to the literature. For the latter, we are interested in exploring how computationally complex its use is in the worst-case scenario. This kind of question, first addressed by \citet{EndrissEtAlJAIR2012}, is regularly asked in the literature of judgment aggregation \citep{baumeister2013computational,lang2014hard,EndrissDeHaanAAMAS2015}. As \citet{EndrissEtAlJAIR2020} wrote recently, the problem of determining the collective outcome of a given judgment aggregation rule  is ``the most fundamental algorithmic challenge in this context''. 
	
	The remainder of this paper is organised as follows. Section~\ref{sec:model} reviews the basic model of judgment aggregation, while Section~\ref{sec:axioms} introduces our two original axioms of egalitarianism and the rules they induce. Section~\ref{sec:manip} analyses the relationship between egalitarianism and strategic manipulation in judgment aggregation, and Section~\ref{sec:computation} focuses on relevant computational aspects: although the general problems of outcome determination and of strategic manipulation are proven to be very difficult, we propose a way to confront them with the tools of \emph{Answer Set Programming}~\citep{Gelfond08}.

	\section{Basic Model} \label{sec:model}
	
	Our framework relies on the standard formula-based model of judgment aggregation \citep{LP2002}, but for simplicity we also use notation commonly employed in binary aggregation \citep{GrandiEndrissAAAI2010}.
	
	
	Let $\mathbb{N}$ denote the (countably infinite) set of all agents that can potentially participate in a judgment aggregation setting. In every specific such setting, a finite set of agents $N \subset \mathbb{N}$ of size~$n\geq 2$ express judgments on a finite and nonempty set of \emph{issues} (formulas in propositional logic) $\Phi = \{\phi_1, \dots, \phi_m \}$,  called the \emph{agenda}.  $\mathcal{J}(\Phi) \subseteq  \{0,1 \}^m$  denotes the set of all admissible opinions on~$\Phi$. Then, a \emph{judgment}~$J$ is a vector in $\mathcal{J}(\Phi)$, with 1 (0) in position~$k$ meaning that the issue $\phi_k$ is accepted (rejected). $\overline{J}$ is the \emph{antipodal} judgment of~$J$: for all $\phi \in \Phi$, $\phi$ is accepted in~$\overline{J}$ if and only if it is rejected in~$J$. 
	
	A \emph{profile} $\Pf = (J_1, \dots J_n) \in \mathcal{J}(\Phi)^n$ is a vector of individual judgments, one for each agent in a group~$N$. We write $\Pf' =_{-i} \Pf$  when the profiles~$\Pf$ and~$\Pf'$ are the same, besides the judgment of agent~$i$. We write $\Pf_{-i}$ to denote the profile $\Pf$ with agent $i$'s judgment removed, and $(\Pf, J) \in \mathcal{J}(\Phi)^{n+1}$ to denote the profile $\Pf$ with judgment $J$ added. 
	A \emph{judgment aggregation rule}~$F$ is a function that  maps every possible profile~$\Pf \in \mathcal{J}(\Phi)^n$, for every group~$N$ and agenda~$\Phi$, to a nonempty set $F(\Pf)$ of collective judgments  in~$\mathcal{J}(\Phi)$. Note that a judgment aggregation rule is defined over groups and agendas of variable size, and  may return several, tied, collective judgments.
	
	The agents that participate in a judgment aggregation scenario will naturally have preferences over the outcome produced by the aggregation rule. First, given an agent~$i$'s truthful judgment~$J_i$,  we need to determine when agent~$i$ would prefer a judgment~$J$ over a different judgment~$J'$. The most prevalent type of such preferences considered in the judgment aggregation literature is that of \emph{Hamming distance} preferences~\citep{baumeister2015complexity,baumeister2017strategic,TerzopoulouEndrissAAAI2018,BotanEndrissAAMAS2020}.
	
	The Hamming distance between two judgments $J$ and $J'$ equals the number of issues on which these judgments disagree---concretely, it is defined as $H(J,J')= \sum_{\phi \in \Phi} |J(\phi) - J'(\phi)|$, where $J(\phi)$ denotes the binary value in the position of~$\phi$ in~$J$. For example, $H(100, 111)=2$. Then, the (weak, and analogously strict) preference of agent~$i$ over judgments is defined by the relation $\succeq_i$ (where $J \succeq_i J'$ means that $i$'s utility from $J$ is higher than that from~$J'$):
	\[ J \succeq_i J' \text{ if and only if } H(J_i, J) \leq H(J_i, J').\]
	
	\noindent  But an aggregation rule often outputs more than one judgment, and thus we also need to determine agents' preferences over sets of judgments.\footnote{Various approaches have been taken within the area of social choice theory in order to extend preferences over objects  to preferences over sets of objects 
		---see \citet{barbera2004ranking} for a review.} 
	We define two requirements guaranteeing that the preferences of the agents over sets of judgments are consistent with their preferences over single judgments.  To that end, let $\setsucceq_i$ (with strict part $\setsucc_i$) denote agent $i$'s preferences over sets $X,Y \subseteq \mathcal{J}(\Phi)$. We require that $\setsucceq_i$ is related to $\succeq_i$ as follows: 
	\begin{itemize}
		\item $J \succeq_i J'$  if and only if $\{J\} \setsucceq_i \{J'\}$, for any $J,J' \in \mathcal{J}(\Phi)$;
		\item $X \setsucc_i Y$ implies that there exist some $J \in X$ and $J' \in Y$ such that $J \succ_i J'$ and $\{J,J'\} \not\subseteq X \cap Y$. 
	\end{itemize}
	
	\noindent The above conditions hold for almost all well-known preference extensions. For example, they hold for the \emph{pessimistic} preference ($X \pess Y$ if and only if there exists $J'\in Y$ such that $J \succ J'$ for all $J \in X$) and the \emph{optimistic} preference ($X \opt Y$ if and only if there exists $J\in X$ such that $J \succ J'$ for all $J' \in Y$) of \citet{duggan2000strategic}, as well as the preference extensions of \citet{gardenfors1976manipulation} and \citet{kelly1977strategy}. The results provided in this paper abstract away from specific preference extensions.

	
	\section{Egalitarian Axioms and Rules} \label{sec:axioms}
	This section focuses on two axioms of egalitarianism in judgment aggregation. We examine them in relation to each other and to existing properties from belief merging, as well as to the standard majority property defined below. Most of the well-known judgment aggregation rules return the majority opinion, when that opinion is logically consistent~\citep{endrissHB2016}.\footnote{A central problem in judgment aggregation concerns the fact that the issue-wise majority is not always logically consistent \citep{LP2002}.} 
	
	Let $m(\Pf)$ be the judgment that accepts exactly those issues accepted by a strict majority of agents in~$\Pf$. A rule~$F$ is \emph{majoritarian} when for all profiles~$\Pf$, $m(\Pf) \in \mathcal{J}(\Phi)$ implies that $F(\Pf)=\{J\}$.

	Our first axiom with an egalitarian flavour is the \emph{maximin property}, suggesting that we should aim at maximising the utility of those agents that will be worst off in the outcome. Assuming that everyone submits their truthful judgment during the aggregation process, this means that we should try to minimise the distance of the agents that are furthest away from the  outcome. Formally:
	\begin{itemize}
		\item[$\blacktriangleright$] A  rule $F$ satisfies the \textbf{maximin} property if for all profiles~$\Pf \in \mathcal{J}(\Phi)^n$ and judgments $J\in F(\Pf)$ there do not exist judgment $J' \in \mathcal{J}(\Phi)$ and  agent $j\in \N$ such that
		\[  H(J_i,J') < H(J_j,J) \text{ for all } i \in \N.
		\]
	\end{itemize}
	
	
	\noindent Although the maximin property is quite convincing, there are settings  like those motivated in the Introduction where it does not offer sufficient egalitarian guarantees. We thus consider a different property next, which we call the \emph{equity property}. This axiom requires that the gaps in the agents' satisfaction be minimised. In other words, no two agents should find themselves in very different distances with respect to the collective outcome. 
	Formally:
	\begin{itemize}
		\item[$\blacktriangleright$]
		A rule $F$ satisfies the \textbf{equity} property if for all profiles~$\Pf \in \mathcal{J}^n$ and judgments $J\in F(\Pf)$, there do not exist judgment $J' \in \mathcal{J}(\Phi)$ and agents $i',j'\in \N$ such that
		\[   |H(J_{i}, J')- H(J_{j},J')|  < |H(J_{i'}, J) - H(J_{j'}, J)| \text{ for all }  i,j \in N.
		\]
	\end{itemize}
	
	\noindent No rule that satisfies either the maximin- or equity property can be majoritarian.\footnote{This includes popular rules like the median rule~\cite{NPP2014condorcet}---known under a number of other names, notably \emph{distance-based rule}~\cite{PigozziSyn2006}, \emph{Kemeny rule}~\cite{endrissHB2016}, and \emph{prototype rule}~\cite{MillerOshersonSCW2009}.} As an  
	illustration, in a profile of only two agents who disagree on some issues, any egalitarian rule will try to reach a compromise, and this compromise will not be affected if any agents holding one of the two initial judgments are added to the profile---in contrast, 
	a majoritarian rule will simply conform to the crowd. 
	
	Proposition~\ref{prop:minmax-egal} shows that it is also impossible for the maximin property and the equity property to simultaneously hold. Therefore, we have established the logical independence of all three axioms discussed so far: maximin, equity, and majoritarianism.

	\begin{prop}\label{prop:minmax-egal}
		No judgment aggregation rule can satisfy both the maximin property and the equity property. 
	\end{prop}
	
	\begin{proof}
		Take an agenda $\Phi$ where $\mathcal{J}(\Phi)$ consists of the nodes in the graph below and consider the profile~$\Pf=(J_1,J_2)$. Each edge is labelled with the Hamming distance between the judgments.
		
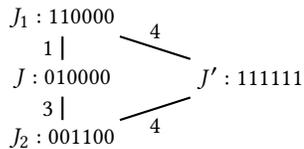
\begin{center}
			\begin{tikzpicture}
			\tikzset{vertex/.style = {shape=circle,draw,minimum size=1.5em}}
			\tikzset{edge/.style = {->,> = latex'}}
			\tikzset{every loop/.style={min distance=15mm,in=130,out=50, looseness=7}}
			
			\node (a) at (0,0) {$J_1: 110000$};
			\node (b) at (0, -0.8) {$J: 010000$};
			\node (c) at (0, -1.6) {$J_2: 001100$};
			\node (d) at (2.5, -0.8) {$J': 111111$};
			
			\path[-, thick] (a) edge  [left] node {$1$} (b);
			\path[-, thick] (b) edge  [left] node {$3$} (c);
			\path[-, thick] (a) edge  [above] node {$4$} (d);
			\path[-, thick] (c) edge  [below] node {$4$} (d);
			\end{tikzpicture}
		\end{center}
		
		\noindent Every aggregation rule satisfying the maximin property will return $\{J\}$, as this judgment maximises the utility of the worst off agent---in this case, agent 2. However every rule satisfying the equity property will return $\{J'\}$, as this judgment minimises the difference in utility between the best off and worst off agents. Thus, there is no rule that can satisfy the two properties at the same time. \end{proof}
	
	\noindent From Proposition~\ref{prop:minmax-egal}, we also know now that the two properties of egalitarianism generate two disjoint classes of aggregation rules. In particular, in this paper we focus on the \emph{maximal} rule that meets each property: a rule~$F$ is the maximal one of a given class if, for every profile~$\Pf$, the outcomes obtained by any other rule in that class are always outcomes of~$F$ too.\footnote{Of course, several natural refinements of these rules can be defrined, with respect to various other axiomatic properties that we may find desirable. Identifying and studying such rules is an interesting direction for future research.} 
	
	The maximal rule satisfying the maximin property is the rule \emph{MaxHam} (see, e.g., \citeauthor{lang2011judgment}, \citeyear{lang2011judgment}). For all profiles~$\Pf  \in \mathcal{J}(\Phi)^n$,
	\[\text{MaxHam}(\Pf) = \argmin_{J \in \mathcal{J}(\Phi)} \max_{i\in\N} H(J_i,J). \]
	
	\noindent Analogously, we define a rule new to the judgment aggregation literature, which is the maximal one satisfying the equity property. For all profiles~$\Pf  \in \mathcal{J}(\Phi)^n$,
	\[\text{MaxEq}(\Pf) = \argmin_{J \in \mathcal{J}(\Phi)} \max_{i,j\in\N} |H(J_i,J)-H(J_j,J)|. \]
	
	\noindent To better understand these rules, consider an agenda with six issues: $p,q,r\equiv p\wedge q$, and their negations. Suppose that there are only two agents in a profile~$\Pf$, holding judgments $J_1=(111)$ and $J_2=(010)$. Then, we have that $\text{MaxHam}(\Pf)=\{(111),(010)\}$, while $\text{MaxEq}=\{(000),(100)\}$. In this example, the difference in spirit between the two rules of our interest is evident. Although the MaxHam rule is able to fully satisfy exactly one of the agents without causing much harm to the other, it still creates greater unbalance than the MaxEq rule, which ensures that the two agents are equally happy with the outcome (under Hamming-distance preferences). In that sense, MaxEq is better suited for a group of agents that do not want any of them to feel particularly put upon, while MaxHam seems more desirable when a minimum level of happiness is asked for.  
	
	MaxHam generalises minimax approval voting \cite{brams2007minimax}, which is the special case without logical constraint on the judgments, meaning agents may approve any subset of issues. \citet{brams2007minimax} show that MaxHam remains manipulable in this special case. As finding the outcome of minimax is computationally hard, \citet{caragiannis2010approximation} provide approximation algorithms that circumvent this problem. They also demonstrate the interplay between manipulability and lower bounds for the approximation algorithm---establishing strategyproofness results for approximations of minimax. 
	\subsection{Relations with Egalitarian Belief Merging}\label{sec:bel-merg}
	A framework closely related to ours is that of belief merging~\cite{KoniecznyP11}, which is concerned with how to aggregate several (possibly inconsistent) sets of beliefs into one consistent belief set.\footnote{We refer to \citet{BMtrends} for a detailed comparison of the two frameworks.} Egalitarian belief merging is studied by \citet{everaere2014egalitarian}, who examine interpretations of the \emph{Sen-Hammond equity condition}~\cite{sen1997choice} and the \emph{Pigou-Dalton transfer principle}~\cite{daltonequity}---two properties that are logically incomparable.\footnote{Another egalitarian property in belief merging is the \emph{arbitration postulate}. We do not go into detail on this postulate, but refer the reader to \citet{KoniecznyP11}.} We situate our egalitarian axioms within the context of these egalitarian axioms from belief merging;
	we reformulate these axioms into our framework. 
	
	\begin{itemize}
		\item[$\blacktriangleright$] Fix an arbitrary profile $\Pf$, agents $i, j$, and any three judgment sets $J, J' \in \mathcal{J}(\Phi)$. An aggregation rule $F$ satisfies the \textbf{Sen-Hammond equity property} if  whenever
		\[H(J_i, J) < H(J_i,J') < H(J_j,J') < H(J_j, J)\] and $H(J_{i'}, J) = H(J_{i'}, J')$ for all other agents $i' \in N\setminus \{i,j\}$, then $J \in F(\Pf)$ implies $J' \in F(\Pf)$.
		
	\end{itemize}
	
	\begin{prop}
		If a rule satisfies either the maximin property or the equity property, then it will satisfy the Sen-Hammond equity property. 
	\end{prop} 
	\begin{proof}[Proof (sketch)]
		Let $\Pf = (J_i, J_j)$ be a profile such that $ H(J, J_i) < H(J',J_i)<H(J',J_j) <H(J, J_j)$, and $H(J_{i'}, J) = H(J_{i'}, J')$ for all other agents $i' \in N\setminus \{i,j\}$. Suppose $F$ satisfies the equity property---if there is some agent $i'$ such that $|H(J_i, J) - H(J_{i'}, J)| > |H(J_i, J) - H(J_j, J)|$, then $J \in F(\Pf)$ if and only if $J' \in F(\Pf)$, as the maximal difference in distance will be the same for the two judgments. If this is not the case, then agents $i$ and $j$ determine the outcome regarding $J$ and $J'$ so clearly $J \in F(\Pf)$ implies $J' \in F(\Pf)$. The argument for other cases proceeds similarly.
		
		If $F$ satisfies the maximin property, then a similar argument tells us that if membership of $J$ and $J'$ in the outcome is determined by an agent other than $i$ or $j$, we will either have both or neither. If $i$, and $j$ are the determining factor then $J \in F(\Pf)$ implies $J' \in F(\Pf)$. 
	\end{proof}
	
	\begin{itemize} 
		\item[$\blacktriangleright$] Given a profile $\Pf =(J_1, \ldots , J_n)$ and agents $i$ and $j$ such that:
		\begin{itemize}
			\item $H(J_i, J) < H(J_i,J') \leq H(J_j,J') <H(J_j, J)$,  
			\item $ H(J_i,J') - H(J_i,J) = H(J_j,J') - H(J_j,J)$, and
			\item $H(J_{i^*}, J) = H(J_{i^*}, J')$ for all other agents $i^* \in N\setminus \{i,j\}$,
		\end{itemize}
		\noindent $F$ satisfies the \textbf{Pigou-Dalton transfer principle} if $J' \in F(\Pf)$ implies $J \not\in F(\Pf)$.
	\end{itemize}
	We refer to these axioms simply as \emph{Sen-Hammond}, and \emph{Pigou-Dalton}. Note that Pigou-Dalton is also a weaker version of our equity property, as it stipulates that the difference between utility in agents should be lessened under certain conditions, while the equity property always aims to minimise this distance. 
	While we can find a rule that satisfies both the equity property and a weakening of the maximin property, Sen-Hammond, we cannot do the same by weakening the equity property. 
	
	\begin{prop}
		No judgment aggregation rule can satisfy both the maximin property and Pigou-Dalton. 
	\end{prop}
	\begin{proof} Consider the domain $\mathcal{J}(\Phi) = \{ J_1, J_2, J_3, J, J'\}$ with the following Hamming distances between judgment sets.\footnote{One such domain would be the following, where $J = 00000000111$, $J' = 00000001110$, $J_1 = 00000010011$, $J_2 = 00000111000$, and $J_3 = 11111001111$.} 
		\begin{center}		
			\begin{tabular}{lccccc}
				& $J$ & $J'$ & $J_1$ & $J_2$ & $J_3$\\\hline
				$J_1$ & 2 & 4 & 0 & 4 & 8\\  
				$J_2$ & \cellcolor[gray]{0.9}6 & 4 & 4 & 0 & \cellcolor[gray]{0.9}10\\  
				$J_3$ & \cellcolor[gray]{0.9}6 & \cellcolor[gray]{0.9}6 & \cellcolor[gray]{0.9}8 & \cellcolor[gray]{0.9}10 & 0\\  
			\end{tabular}
		\end{center}
		\noindent Let $\Pf = (J_1, J_2, J_3)$. If $F$ satisfies the maximin property, $\{J, J'\} \subseteq F(\Pf)$, as we can see from the grey cells. This means Pigou-Dalton is violated in this profile, as $J' \in F(\Pf)$ should imply $J \not\in F(\Pf)$. 
	\end{proof}

	\noindent	We summarise the observations of this section in Figure~\ref{fig:axioms}.
	
	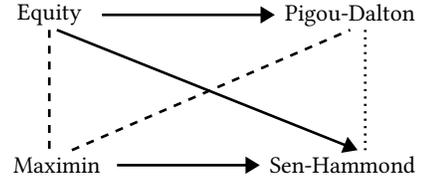
\begin{figure}
		\begin{tikzpicture}
		\tikzset{vertex/.style = {shape=circle,draw,minimum size=1.5em}}
		\tikzset{edge/.style = {->,> = latex'}}
		\tikzset{every loop/.style={min distance=15mm,in=130,out=50, looseness=7}}
		
		\node (EQ) at (0,0) {Equity}; 
		\node (PD) at (4,0) {Pigou-Dalton}; 
		
		\node (MM) at (0.1,-2) {Maximin}; 
		\node (SH) at (3.9,-2) {Sen-Hammond};
		
		\path[-Triangle, line width=1] (0.7,0) edge  [right] node [near start] {} (3,0);
		
		\path[-Triangle, line width=1] (0.9,-2) edge  [right] node [near start] {} (2.8,-2);
		
		\path[-Triangle, line width=1] (0.1,-0.2) edge  [right] node [near start] {} (4.1,-1.8);

		\path[-, thick, dashed, line width=1] (0,-0.2) edge  [right] node [near start] {} (0,-1.8);
		
		\path[-, thick, dotted, line width=1] (4.2,-0.2) edge  [right] node [near start] {} (4.2,-1.8);
		
		\path[-, thick, dashed, line width=1] (0.3,-1.8) edge  [right] node [near start] {} (4,-0.2);
		
		\end{tikzpicture}
		\caption{Dashed lines denote incompatibility, dotted lines incomparability, and arrows implication relations.}\label{fig:axioms}.
	\end{figure}
	
	
	\section{Strategic Manipulation}\label{sec:manip}
	
	This section provides an account of strategic manipulation with respect to the egalitarian axioms defined in Section~\ref{sec:axioms}.
	We start off with presenting the most general notion of strategic manipulation in judgment aggregation, introduced by \citet{dietrich2007strategy}.\footnote{The original definition of \citet{dietrich2007strategy} concerned single-judgment collective outcomes, and a type of preferences that covers Hamming-distance ones. } We assume Hamming preferences throughout this section. 
	
	\begin{defn} \label{defn:sp}
		A rule~$F$ is susceptible to \textbf{manipulation} by agent~$i$ in profile~$\Pf$, if there exists a profile $\Pf' =_{-i} \Pf$ such that $F(\Pf') \setsucc_i F(\Pf)$.
	\end{defn}
	\noindent We say that $F$ is \emph{strategyproof} in case $F$ is not manipulable by any agent $i \in N$ in any profile $\Pf \in \mathcal{J}(\Phi)^n$. 
	
	Proposition~\ref{prop:sp-imposs} shows an important fact: In judgment aggregation, egalitarianism is incompatible with strategyproofness.\footnote{This in in line with \citeauthor{brams2007minimax}'s work on the minimax rule in approval voting.}

	\begin{prop} \label{prop:sp-imposs}
		If an aggregation rule is strategyproof, it cannot satisfy the maximin property or the equity property. 
	\end{prop}
	
	\begin{proof}
		We show the contrapositive. Let $\Phi$ be an agenda such that $\mathcal{J}(\Phi) = \{000000, 110000, 111000, 111111\}$. Consider the following two profiles $\Pf$ (left) and $\Pf'$ (right).
		\smallskip 
		
		\begin{minipage}{0.22\textwidth}
			
			\begin{center}
				\begin{tabular}{c c} 
					$J_i$ & $111 000$ \\ 
					$J_j$ & $000 000$  \\ 
					\midrule
					$F(\Pf)$ & $110 000$ \\ 
				\end{tabular}
			\end{center}
		\end{minipage}
		\begin{minipage}{0.22\textwidth}
			\begin{center}
				\begin{tabular}{c c} 
					$J'_i$ & $111 111$ \\ 
					$J'_j$ & $000 000$  \\ 
					\midrule
					$F(\Pf')$ & $111 000$ \\ 
				\end{tabular}
			\end{center}
		\end{minipage}
		\smallskip 
		
		\noindent In profile~$\Pf$, both the maximin and the equity properties prescribe that $110000$ should be returned as the single outcome, while in profile~$\Pf'$ they agree on $111000$. Because $\Pf' = (\Pf_{-i}, J'_i)$, and $111000 \succ_i 110000$, this is a successful manipulation. Thus, if $F$ satisfies the maximin or the equity property, it fails strategyproofness. 
	\end{proof}
	
	\noindent Strategyproofness according to Definition~\ref{defn:sp} is a strong requirement, which many known 
	rules fail \citep{BotanEndrissAAMAS2020}. 
	We investigate two more nuanced notions of strategyproofness that are novel to judgment aggregation, yet have familiar counterparts in voting theory. 
	
	First, \emph{no-show manipulation} happens when an agent can achieve a preferable outcome simply by not submitting any judgment, instead of reporting a truthful or an untruthful one. 
	
	\begin{defn}
		A rule $F$ is susceptible to \textbf{no-show manipulation} by agent~$i$ in profile~$\Pf$ if $ F(\Pf_{-i}) \setsucc_i  F(\Pf)$.
	\end{defn} 
	
	\noindent We say that $F$ satisfies \emph{participation} if it is not susceptible to no-show manipulation by any agent $i \in N$ in any profile.\footnote{cf.\ the no-show paradox in voting \citep{fishburn1983paradoxes}.} 
	
	Second, \emph{antipodal strategyproofness} poses another barrier against manipulation, by stipulating that an agent cannot change the outcome towards a better one for herself by reporting a totally untruthful judgment. This is a strictly weaker requirement than full strategyproofness, serving as a protection against excessive lying.

	\begin{defn}
		A rule $F$ is susceptible to \textbf{antipodal manipulation} by agent~$i$ in profile~$\Pf$ if $F(\Pf_{-i},\overline{J_i}) \setsucc_i F(\Pf)$.
	\end{defn}
	
	\noindent We say that $F$ satisfies \emph{antipodal strategyproofness} if it not susceptible to antipodal manipulation by any agent $i \in N$ in any profile. As is the case for participation, antipodal strategyproofness is a weaker notion of strategyproofness as far as the MaxHam and the MaxEq rules are concerned.
	
	In voting theory, \citet{sanver2009one} show that participation implies antipodal strategyproofness (or \emph{half-way monotonicity}, as called in that framework) for rules that output a single winning alternative. Notably, this is not always the case in our model (see Example~\ref{ex: part-sp-ex}). This is not surprising, as obtaining such a result independently of the preference extension would be significantly stronger than the result by \citet{sanver2009one}.  We are, however, able to reproduce this relationship between participation and strategyproofness in Theorem~\ref{thm:part-sp-impl}, for a specific type of preferences.

	\begin{ex}  \label{ex: part-sp-ex}
		We present a rule that satisfies participation but violates antipodal strategyproofness. The other direction admits a similar example, and is thus omitted. Note that the rule demonstrated is quite unnatural for simplicity of the presentation.
		
		Consider an agenda~$\Phi$ with  $\mathcal{J}(\Phi)=\{00,01,11\}$.\footnote{For other agendas we can simply take the rule to be constant.} We construct an anonymous rule~$F$ that is only sensitive to \emph{which} judgments are submitted and not to their quantity:
		
		\vspace{0.1cm}
		
		\noindent 
		$F(00)=F(11)=F(01,00)=F(00,11)=\{01,11\}$;\\
		$F(01)=\{00,11\}$; $F(01,11)=F(01,00,11)=\{01\}$.
		
		\vspace{0.1cm}
		
		\noindent For the pessimistic preference,  no agent can be strictly better off by abstaining. However, compare the profiles $(01,00)$ and $(01,11)$: agent~2 with truthful judgment~$00$ can move from outcome $\{01,11\}$ to outcome~$\{01\}$, which is strictly better for her.
	\end{ex}
	
	\noindent While the two axioms are independent in the general case, participation implies antipodal strategyproofness (Theorem~\ref{thm:part-sp-impl}) if we stipulate that
	
	\begin{itemize}
		\item $X \setsucc_i Y$ \textbf{if and only if} there exist some $J \in X$ and $J' \in Y$ such that $J \succ_i J'$ and $\{J,J'\} \not\subseteq X \cap Y$. 
	\end{itemize}
	
	\noindent If a preference satisfies the above condition, we say that it is \emph{decisive}. This condition gives rise to a preference extension equivalent to the \emph{large preference extension} of \citet{KrugerTerzopoulouAAMAS2020}. Note that a decisive preference is not necessarily acyclic---in fact, it may even be symmetric. The interpretation of such a preference extension is slightly different than the usual one; when we say that a rule is strategyproof for a decisive preference where both $J \setsucc J'$ and $J' \setsucc J$ hold, we mean that no agent~$i$ with  $J \setsucc_i J'$ and no agent~$j\neq i$ with $J' \setsucc_j J$ will ever have an incentive to manipulate.
	
	Using Lemma~\ref{lem:maxmin}, we can now prove a result analogous to the one in voting theory, to give a complete picture of how these axioms relate to each other in judgment aggregation. 
	
	\begin{lem}\label{lem:maxmin}
		For judgment sets $J, J'$ and $J''$:  $H(J, J') > H(J,J'')$, if and only if $H(\overline{J}, J') < H(\overline{J},J'')$. 
	\end{lem}
	
	\begin{proof}\label{prop:maximin}
		For judgment sets $J, J' \in \mathcal{J}(\Phi)$, $H(\overline{J}, J') = m - H(J, J')$. Suppose $H(J, J') > H(J,J'')$. Then $H(\overline{J}, J') = m - H(J, J') < m - H(J,J'') = H(\overline{J}, J'')$. The other direction is analogous. 
	\end{proof}
	
	\begin{thm} \label{thm:part-sp-impl}
		For decisive preferences over sets of judgments, participation implies antipodal strategyproofness. 
	\end{thm}
	
	\begin{proof}
		Working on the contrapositive, suppose that $F$ is susceptible to antipodal manipulation. We will prove that $F$ is susceptible to no-show manipulation too.  We know that there exists $i \in N$ such that $F(\Pf_{-i},\overline{J_i}) \setsucc_i F(\Pf_{-i}, J_i)$, for some profile~$\Pf$. This means that there exist  $J' \in F(\Pf_{-i},\overline{J_i})$ and $J \in F(\Pf_{-i}, J_i)$ with $J' \succ_i J$. Equivalently,
		\begin{equation} \label{eq:1}
		H(J_i,J') < H(J_i, J)
		\end{equation}
		\noindent Next, consider a judgment~$J'' \in F(\Pf_{-i})$.
		
		If $H(\overline{J_i},J'') < H(\overline{J_i}, J')$, then $F$ is susceptible to no-show manipulation by agent~$i$ in the profile~$(\Pf_{-i}, \overline{J_i})$.  
		
		Otherwise, $H(\overline{J_i},J') \leq H(\overline{J_i}, J'')$. Then Lemma~\ref{lem:maxmin} implies that $H(J_i,J'') \leq H(J_i, J')$. So, together with Inequality~\eqref{eq:1}, we have that
		$H(J_i,J'') < H(J_i, J)$.
		This  means that $F$ is susceptible to no-show manipulation by agent~$i$ in the profile~$(\Pf_{-i}, J_i)$.  
	\end{proof}

	\noindent We next prove that any rule satisfying the maximin property is immune to both no-show manipulation and antipodal manipulation (Theorem~\ref{thm:max-asp}), while this is not true for the equity property (Proposition~\ref{prop:eq-asp}).\footnote{Note that antipodal strategyproofness is not so weak a requirement that is immediately satisfied by all ``utilitarian'' aggregation rules.  For example, the Copeland voting rule fails the analogous axiom of half-way monotonicity~\citep{Zwicker16}.} We emphasise that the theorem holds for \emph{all} preference extensions. These results---holding for two independent notions of strategyproofness---are significant for two reasons. First, they bring to light the conditions under which we can have our cake and eat it too, simultaneously satisfying an egalitarian property and a degree of strategyproofness. In addition, they provide a further way to distinguish between the properties of maximin and equity: the former is better suited in contexts where we may worry about the agents' strategic behaviour.
	
	\begin{thm}\label{thm:max-asp}
		The maximin property implies participation and antipodal strategyproofness. 
	\end{thm}
	
	\begin{proof}
		We prove the participation case; the proof for antipodal strategyproofness is analogous, and utilises Lemma~\ref{lem:maxmin}.
		
		Suppose for contradiction that $F$ is a rule that satisfies the maximin property but violates participation. Then there must exist  agent~$i \in N$ and profile~$\Pf$ where $J_i$ is agent $i$'s truthful judgment, such that $F(\Pf_{-i}) \setsucc_i F(\Pf)$. This means there must exist judgments $J \in F(\Pf)$ and $J' \in F(\Pf_{-i})$ such that $J' \succ_i J$ and $\{J, J'\} \not\subseteq F(\Pf) \cap F(\Pf_{-i})$. Because agent~$i$ strictly prefers $J'$ to $J$, this means that $H(J_i, J) > H(J_i, J')$. We consider two cases. 
		
		\underline{Case 1:} Suppose that $J' \not\in F(\Pf)$. Let $k$ be the distance between the worst off agent's judgment in~$\Pf$ and any judgment in $F(\Pf)$. Then,
		\begin{equation} \label{eq:ineq-cond-part1}
		H(J_{j'}, J) \leq k \text{ for all } j' \in N.
		\end{equation}
		\noindent   We know that $H(J_i, J') < k$ because $H(J_i, J) \leq k$, and agent~$i$ strictly prefers $J'$ to $J$. From Inequality~\eqref{eq:ineq-cond-part1}, this means that if $J'$ is not among the outcomes in $F(\Pf)$, there has to be some $j \in N \setminus \{i\}$ such that $H(J_j, J') > k$. But all judgments submitted to profile~$(\Pf_{-i})$ by agents in $N \setminus \{i\}$ are at most at distance $k$ from $J$ by Inequality~\eqref{eq:ineq-cond-part1}, so $J$ would be selected by any rule satisfying the maximin property will select~$J$ as an outcome of $F(\Pf_{-i})$---instead of~$J'$, a contradiction. 
		
		\underline{Case 2:} Suppose that $J' \in F(\Pf)$, meaning that $J \not\in F(\Pf_{-i})$. Analogously to the first case, let $k'$ be the distance between the worst off agent's judgment in~$\Pf_{-i}$ and any judgment in $F(\Pf_{-i})$. Then,
		\begin{equation} \label{eq:ineq-cond-part2}
		H(J_{j'}, J') \leq k' \text{ for all } j' \in N\setminus \{i\}.
		\end{equation}
		\noindent Moreover, since $J \not\in F(\Pf_{-i})$, it is the case that
		\begin{equation} \label{eq:ineq-cond-part3}
		H(J_{j},J) > k' \text{ for some } j \neq i.
		\end{equation}
		\noindent  In profile~$\Pf$, Inequalities~\eqref{eq:ineq-cond-part2} and~\eqref{eq:ineq-cond-part3} still hold.  In addition, we have that $H(J_i, J) > H(J_i, J')$ because agent~$i$ strictly prefers $J'$ to $J$. So, for any rule satisfying the maximin property, judgment~$J'$ will be better as an outcome of $F(\Pf)$ than~$J$, a contradiction.  
	\end{proof} 
	
	\begin{cor}
		The rule MaxHam satisfies antipodal strategyproofness and participation.
	\end{cor}

	\begin{prop} \label{prop:eq-asp}
		No rule that satisfies the equity property can satisfy participation or antipodal strategyproofness .
	\end{prop}
	
	\begin{proof}
		The following is a counterexample for antipodal strategyproofness. A similar one exists for participation.
		
		Consider the following profiles $\Pf = \{Ji, J_j\}$ and $\Pf' = (\Pf_{-i}, \overline{J_i})$. We give a visual representation of the profiles as well as the outcomes under an arbitrary rule~$F$ that satisfies the equity principle. We specify that $\mathcal{J}(\Phi) = \{00110, 00000, 01110, 10000, 11111\}$.

		\begin{center}
			\begin{tikzpicture}
			\tikzset{vertex/.style = {shape=circle,draw,minimum size=1.5em}}
			\tikzset{edge/.style = {->,> = latex'}}
			\tikzset{every loop/.style={min distance=15mm,in=130,out=50, looseness=7}}

			\node (fj) at (-1.7, 0.1) {$F(\Pf)$};
			\node (f) at (-1.7, -0.25) {$00110$};

			\node (j1) at (0.2,0.4) {$\Pf$};
			\node (d) at (0,0) {$J_i: 00000$};
			\node (e) at (0, -0.5) {$J_j: 01110$};
			
			\node (fj') at (2, 0.1) {$F(\Pf')$};
			\node (c) at (2, -0.25) {$10000$};

			\node (j1bar) at (4.2,0.4) {$\Pf'$};
			\node (a) at (4,0) {$J'_i: 11111$};
			\node (b) at (4, -0.5) {$ J'_j: 01110$};

			\path[-, thick] (a) edge  [above] node {$4$} (c);
			\path[-, thick] (b) edge  [below] node {$4$} (c);
			
			\path[-, thick] (d) edge  [above] node {$1$} (c);
			\path[-, thick] (e) edge  [below] node {$4$} (c);
			
			\path[-, thick] (d) edge  [above] node {$2$} (f);
			\path[-, thick] (e) edge  [below] node {$1$} (f);
			\end{tikzpicture}
		\end{center}
		
		\noindent Each edge from an individual judgment to a collective one is labelled with the Hamming distance between the two. It is clear that agent~$i$ will benefit from her antipodal manipulation, as her true judgment is much closer to the singleton outcome in $\Pf'$ than the singleton outcome in $\Pf$. 
	\end{proof}
	
	\begin{cor}
		The rule MaxEq does not satisfy participation or antipodal strategyproofness.
	\end{cor}

	
	
	
	
	
	
	
	\section{Computational Aspects}
	\label{sec:computation}
	
	We have discussed two aggregation rules that reflect desirable egalitarian principles---i.e., the MaxHam and MaxEq rules---and examined whether they give agents incentives to misrepresent their truthful judgments. In this section we consider how complex it is, computationally, to employ these rules, and the complexity of determining whether an agent can manipulate the collective outcome.
	
	The MaxHam rule has been considered from a computational perspective 
	before~\cite{DeHaanSlavkovik17,DeHaanSlavkovik19,DeHaan18}.
	Here, we extend this analysis to the MaxEq rule, and we compare the two rules
	with each other on their computational properties.
	Concretely, we primarily establish some computational complexity results;
	motivated by these results, we then illustrate how some computational problems related
	to these rules can be solved using the paradigm of Answer Set Programming.
	
	\subsection{Computational Complexity}
	\label{sec:complexity}
	
	We investigate some computational complexity aspects of the
	judgment aggregation rules that we have considered.
	Due to space constraints, we will only describe the main lines of these results---%
	for full details, we refer to the accompanying
	Appendix.
	
	Consider the problem of outcome determination (for a rule~$F$).
	This is most naturally modelled as a search problem,
	where the input consists of an agenda~$\Phi$
	and a profile~$\Pf = (J_1,\dotsc,J_n) \in \mathcal{J}(\Phi)^{n}$.
	The problem is to produce some judgment set~$J^{*} \in F(\Pf)$.
	We will show that for the MaxEq rule, this problem can be solved in polynomial time
	with a logarithmic number of calls to an oracle for \NP{} search problems
	(where the oracle also produces a witness for yes answers---also called
	an \emph{\FNP{} witness oracle}).
	Said differently, the outcome determination problem for the the MaxEq rule
	lies in the complexity class \FPNPlogwit{}.
	We also show that the problem is complete for this class
	(using the standard type of reductions used for search problems:
	polynomial-time Levin reductions).
	
	\begin{thm}
		\label{thm:fpnplogwit-completeness}
		The outcome determination problem for the MaxEq rule
		is~$\FPNPlogwit{}$-complete under polynomial-time Levin reductions.
	\end{thm}
	\begin{proof}[Proof (sketch)]
		Membership in \FPNPlogwit{} can be shown by giving a polynomial-time algorithm
		that solves the problem by querying an \FNP{} witness oracle a logarithmic number of times.
		The algorithm first finds the minimum value~$k$ of~$\max_{J',J'' \in \Pf} |H(J,J') - H(J,J'')|$
		by means of binary search---requiring a logarithmic number of oracle queries.
		Then, with one additional oracle query, the algorithm can produce
		some~$J^{*} \in \mathcal{J}(\Phi)$ with~$\max_{J',J'' \in \Pf} |H(J^{*},J') - H(J^{*},J'')| = k$.
		
		To show $\FPNPlogwit{}$-hardness,
		we reduce from the problem of finding
		a satisfying assignment of a (satisfiable) propositional formula~$\psi$
		that sets a maximum number of variables to true
		\cite{ChenToda95,KoeblerThierauf90}.
		This reduction works roughly as follows.
		Firstly, we produce 3CNF formulas~$\psi_1,\dotsc,\psi_v$
		where each~$\psi_i$ is 1-in-3-satisfiable if and only if there exists
		a satisfying assignment of~$\psi$ that sets at least~$i$ variables to true.
		Then, for each~$i$, we transform~$\psi_i$ to an agenda~$\Phi_i$
		and a profile~$\Pf_i$ such that there is a judgment set with equal Hamming
		distance to each~$J \in \Pf_i$ if and only if~$\psi_i$ is 1-in-3-satisfiable.
		Finally, we put the agendas~$\Phi_i$ and profiles~$\Pf_i$ together
		into a single agenda~$\Phi$ and a single profile~$\Pf$
		such that we can---from the outcomes selected by the MaxEq rule---%
		read off the largest~$i$ for which~$\psi_i$ is 1-in-3-satisfiable,
		and thus, the maximum number of variables set to true in any
		truth assignment satisfying~$\psi$.
		This last step involves duplicating issues in~$\Phi_1,\dotsc,\Phi_v$
		different numbers of times, and creating logical dependencies between them.
		Moreover, we do this in such a way that from any outcome selected by
		the MaxEq rule, we can reconstruct a truth assignment satisfying~$\psi$
		that sets a maximum number of variables to true.
	\end{proof}
	
	\noindent The result of Theorem~\ref{thm:fpnplogwit-completeness}
	means that the computational complexity of computing outcomes for the MaxEq
	rule lies at the \ThetaP{2}-level of the Polynomial Hierarchy.
	This is in line with previous results on the computational complexity of the outcome
	determination problem for the MaxHam rule---%
	\citet{DeHaanSlavkovik17} showed that a decision variant of the outcome determination problem
	for the MaxHam rule is \ThetaP{2}-complete.
	Notably, our proof (presented in detail in the Appendix) brings out an intriguing fact about        
	a problem that is at first glance simpler than outcome determination for MaxEq:  Given an agenda~$\Phi$
	and a profile~$\Pf$, deciding whether the minimum value
	of~$\max_{i,j\in\N} |H(J_i,J)-H(J_j,J)|$ for~$J \in \mathcal{J}(\Phi)$---%
	the value that the MaxEq rule minimizes---%
	is divisible by~4,
	is \ThetaP{2}-complete (Proposition~\ref{prop:pnplog-completeness}). Intuitively, merely computing the minimum value  that  
	is  relevant for MaxEq  is \ThetaP{2}-hard.
	
	\begin{prop} 
		\label{prop:pnplog-completeness}
		Given an agenda~$\Phi$ and a profile~$\Pf$,
		deciding whether the minimal value
		of~$\max_{J',J'' \in \Pf} |H(J^{*},J') - H(J^{*},J'')|$
		for~$J^{*} \in \mathcal{J}(\Phi)$,
		is divisible by~4,
		is a \ThetaP{2}-complete problem.
	\end{prop}
	
	\noindent Interestingly, we found that the problem of deciding if there exists
	a judgment set~$J^{*} \in \mathcal{J}(\Phi)$ that has the exact same
	Hamming distance to each judgment set in the profile
	is \NP{}-hard, even when the agenda consists of logically independent issues.
	
	\begin{prop}
		\label{prop:np-completeness}
		Given an agenda~$\Phi$ and a profile~$\Pf$,
		the problem of deciding whether there is some~$J^{*} \in \mathcal{J}(\Phi)$
		with~$\max_{J',J'' \in \Pf} |H(J^{*},J') - H(J^{*},J'')| = 0$
		is \NP{}-complete.
		Moreover, \NP{}-hardness holds even for the case where~$\Phi$
		consists of logically independent issues---i.e., the case
		where~$\mathcal{J}(\Phi) = \{0,1\}^m$ for some~$m$.
	\end{prop}
	
	\noindent This is also in line with previous results for the MaxHam rule---%
	\citet{DeHaan18} showed that computing outcomes for the MaxHam rule
	is computationally intractable even when the agenda consists of logically
	independent issues.
	
	Next, we turn our attention to the problem of strategic manipulation.
	Specifically,
	we show that---for the case of decisive preferences over sets of judgment sets---%
	the problem of deciding if an agent~$i$ can strategically manipulate
	is in the complexity class \SigmaP{2}.
	
	\begin{prop}
		\label{prop:manipulation-sigmap2}
		Let~${\succeq}$ be a preference relation over judgment sets that is polynomial-time
		computable, and let~${\setsucceq}$ be a decisive extension over sets of judgment sets.
		Then the problem of deciding if a given agent~$i$ can strategically manipulate
		under the MaxEq rule---%
		i.e.,~given~$\Phi$ and~$\Pf$, deciding if there exists some~$\Pf' =_{-i} \Pf$
		with~$\text{MaxEq}(\Pf') \setsucc_i \text{MaxEq}(\Pf)$---%
		is in the complexity class \SigmaP{2}.
	\end{prop}
	\begin{proof}[Proof (sketch)]
		To show membership in $\SigmaP{2} = \NP^{\NP}$, we describe a nondeterministic
		polynomial-time algorithm with access to an \NP{} oracle that solves the problem.
		The algorithm firstly guesses a new judgment set~$J'_i$ for agent~$i$ in the new
		profile~$\Pf'$, and guesses a truth assignment witnessing that~$J'_i$ is consistent.
		Then, using the \NP{} oracle, it computes the
		values~$k = \max_{J',J'' \in \Pf} |H(J,J') - H(J,J'')|$
		and~$k' = \max_{J',J'' \in \Pf'} |H(J,J') - H(J,J'')|$,
		for~$J \in \mathcal{J}(\Phi)$.
		Finally, it guesses some~$J,J' \in \mathcal{J}(\Phi)$, together with truth assignments
		witnessing consistency, and it verifies that~$J' \succ_i J$,
		that~$J' \in \text{MaxEq}(\Pf')$, that~$J \in \text{MaxEq}(\Pf)$,
		and that~$\{ J, J' \} \not\subseteq \text{MaxEq}(\Pf) \cap \text{MaxEq}(\Pf')$.
		Since these final checks can all be done in polynomial time---%
		using the previously guessed and computed information---%
		one can verify that this can be implemented by an~$\NP^{\NP}$ algorithm.
	\end{proof}
	
	\noindent This \SigmaP{2}-membership result can straightforwardly be extended to
	other variants of the manipulation problem (e.g., no-show manipulation
	and antipodal manipulation) and to other preferences, as well as to the MaxHam rule.
	Due to space constraints, we omit further details on this. Still, we shall mention that  results demonstrating that
	strategic manipulation is very complex are generally more welcome than analogous ones regarding outcome determination.    
	If manipulation is considered a negative side-effect of the agents' strategic behaviour, knowing that it is hard for the  
	agents to materialise it is good news.\footnote{Note though that hardness results regarding manipulation of our egalitarian
		rules remain an open question.} In Section~\ref{sec:asp-maxeq} we will revisit these concerns from a different angle.
	
	\subsection{ASP Encoding for the MaxEq Rule}
	\label{sec:asp-maxeq}
	
	The 
	complexity results in Section~\ref{sec:complexity} leave no doubt that applying our egalitarian rules is 
	computationally difficult. Nevertheless, they also indicate
	that a useful approach for computing outcomes of the MaxEq rule in practice
	would be to encode this problem into the paradigm of Answer Set Programming (ASP)
	\cite{Gelfond08}, and to use ASP solving algorithms.
	ASP offers an expressive automated reasoning framework that typically works
	well for problems at the \ThetaP{2} level of the Polynomial Hierarchy.
	In this section, we will show how this encoding can be done---%
	similarly to an ASP encoding for the MaxHam rule
	\cite{DeHaanSlavkovik19}.
	Due to space restrictions, we refer to the literature
	for details on the syntax and semantics of ASP---%
	e.g.,~\cite{Gelfond08,GebserKaminskiKaufmannSchaub12}.
	
	We use the same basic setup that \citet{DeHaanSlavkovik19} use to represent
	judgment aggregation scenarios---with some simplifications and modifications
	for the sake of readability. In particular,
	we use the predicate \inlinecode{voter/1} to represent individuals,
	we use \inlinecode{issue/1} to represent issues in the agenda,
	and we use \inlinecode{js/2} to represent judgment sets---%
	both for the individual voters and for a dedicated agent \inlinecode{col}
	that represents the outcome of the rule.
	
	With this encoding of judgment aggregation scenarios,
	one can add further constraints on the predicate \inlinecode{js/2}
	that express which judgment sets are consistent, based on the logical relations
	between the issues in the agenda~$\Phi$---as done by \citet{DeHaanSlavkovik19}.
	We refer to their work for further details on how this can be done.
	
	Now, we show how to encode the MaxEq rule into ASP,
	similarly to the encoding of the MaxHam rule by \citet{DeHaanSlavkovik19}.
	We begin by defining a predicate \inlinecode{dist/2} to capture
	the Hamming distance~\inlinecode{D} between the outcome and the judgment
	set of an agent~\inlinecode{A}.
	\StartLineAt{1}\ContinueLineNumber
	\begin{lstlisting}
	dist(A,D) :- voter(A),                                D = #count { X : issue(X), js(col,X), js(A,-X) }. |\label{line:dist}|
	\end{lstlisting}
	Then, we define predicates \inlinecode{maxdist/1}, \inlinecode{mindist/1}
	and \inlinecode{inequity/1} that capture the maximum Hamming distance from the
	outcome to any judgment set in the profile, the minimum such Hamming distance,
	and the difference between the maximum and minimum (or \emph{inequity}), respectively.
	\ContinueLineNumber
	\begin{lstlisting}
	maxdist(Max) :- Max = #max { D : dist(A,D) }. |\label{line:maxdist}|
	mindist(Min) :- Min = #min { D : dist(A,D) }. |\label{line:mindist}|
	inequity(Max-Min) :- maxdist(Max), mindist(Min). |\label{line:inequity}|
	\end{lstlisting}
	Finally, we add an optimization constraint that states that only outcomes
	should be selected that minimize the inequity.%
	\footnote{The expression ``\inlinecode{@30}'' in Line~\ref{line:minimize} indicates the
		priority level of this optimization statement (we used the arbitrary value of~30, and priority levels lexicographically).}
	\ContinueLineNumber
	\begin{lstlisting}
	#minimize { I@30 : inequity(I) }. |\label{line:minimize}|
	\end{lstlisting}
	
	For any answer set program that encodes a judgment aggregation setting,
	combined with Lines~\ref{line:dist}--\ref{line:minimize},
	it then holds that the optimal answer sets are in one-to-one
	correspondence with the outcomes selected by the MaxEq rule.
	
	Interestingly, we can readily modify this encoding to capture refinements
	of the MaxEq rule.
	An example of this is the refinement that selects (among the outcomes of the MaxEq rule)
	the outcomes that minimize the maximum Hamming distance to any judgment set in the profile.
	We can encode this example refinement by adding the following optimization statement
	that works at a lower priority level than the optimization in Line~\ref{line:minimize}.
	\ContinueLineNumber
	\begin{lstlisting}
	#minimize { Max@20 : maxdist(Max) }. |\label{line:minimize2}|
	\end{lstlisting}
	
	\subsection{Encoding Strategic Manipulation}

	We now show how to encode the problem of strategic manipulation into ASP. The value of this section's contribution should 
	be viewed from the perspective of the modeller rather than from that of the agents. That is, even if we do not wish for the 
	agents to be able to easily check whether they can be better off by lying, it may be reasonable, given a profile of 
	judgments, to externally determine whether a certain agent can benefit from being untruthful.
	
	The simplest way to achieve this is with the meta-programming techniques
	developed by \citet{GebserKaminskiSchaub11}.
	Their meta-programming approach allows one to additionally
	express optimization statements that are based on subset-minimality,
	and to transform programs with this extended expressivity to standard
	(disjunctive) answer set programs.
	We use this to encode the problem of strategic manipulation.
	
	Due to space reasons, we will not spell out the full ASP encoding needed to do so.
	Instead, we will highlight the main steps, and describe how these fit together.
	We will use the example of MaxEq, but the exact same approach would work for any
	other judgment aggregation rule that can be expressed in ASP efficiently using regular
	(cardinality) optimization constraints---in other words, for all rules for which the outcome
	determination problem lies at the \ThetaP{2} level of the Polynomial Hierarchy.
	Moreover, we will use the example of a decisive preference~$\setsucc$ over sets
	of judgment sets that is based on a polynomial-time computable preference~$\succ$
	over judgment sets.
	The approach can be modified to work with other preferences as well.
	
	We begin by guessing a new judgment set~$J'_i$ for the individual~$i$ that is trying
	to manipulate---and we assume, w.l.o.g., that~$i = 1$.
	\ContinueLineNumber
	\begin{lstlisting}
	voter(prime(1)). |\label{line:guess-manipulation1}|
	1 { js(prime(1),X), js(prime(1),-X) } 1 :- issue(X). |\label{line:guess-manipulation2}|
	\end{lstlisting}
	
	Then, we express the outcomes of the MaxEq rule, both for the non-manipulated profile~$\Pf$
	and for the manipulated profile~$\Pf'$, using the dedicated agents~\inlinecode{col} (for~$\Pf$)
	and~\inlinecode{prime(col)} (for~$\Pf'$).
	This is done exactly as in the encoding of the problem of outcome determination
	(so for the case of MaxEq, as described in Section~\ref{sec:asp-maxeq})---%
	with the difference that optimization is expressed
	in the right format for the meta-programming method of \citet{GebserKaminskiSchaub11}.
	
	We express the following subset-minimality minimization statement
	(at a higher priority level than all other optimization constraints used so far).
	This will ensure that every possible judgment set~$J'_i$ will be considered
	as a subset-minimal solution.
	\ContinueLineNumber
	\begin{lstlisting}
	_criteria(40,1,js(prime(1),X)) :- js(prime(1),S). |\label{line:complex-opt1}|
	_optimize(40,1,incl). |\label{line:complex-opt2}|
	\end{lstlisting}
	
	To encode whether or not the guessed manipulation was successful,
	we have to define a predicate \inlinecode{successful/0} that is true
	if and only if (i)~$J' \succ_i J$ and (ii)~$J$ and~$J'$ are not both selected as outcome
	by the MaxEq rule for both~$\Pf$ and~$\Pf'$,
	where~$J'$ is the outcome encoded by the statements~\inlinecode{js(prime(col),X)}
	and~$J$ is the outcome encoded by the statements~\inlinecode{js(col,X)}.
	Since we assume that~$\succ_i$ is computable in polynomial time, and since
	we can efficiently check using statements in the answer set
	whether~$J$ and~$J'$ are selected by the MaxEq rule for~$\Pf$ and~$\Pf'$,
	we know that we can define the predicate \inlinecode{successful/0} correctly
	and succinctly in our encoding.
	For space reasons, we omit further details on how to do this.
	
	Then, we express another minimization statement (at a lower priority level than all other
	optimization statements used so far), that states that we should make
	\inlinecode{successful} true whenever possible.
	Intuitively, we will use this to filter our guessed manipulations that are unsuccessful.
	\ContinueLineNumber
	\begin{lstlisting}
	unsuccessful :- not successful. |\label{line:complex-opt3}|
	successful :- not unsuccessful.
	_criteria(10,1,unsuccessful) :- unsuccessful.
	_optimize(10,1,card). |\label{line:complex-opt4}|
	\end{lstlisting}
	
	Finally, we feed the answer set program~$P$ that we constructed so far
	into the meta-programming method, resulting in a new (disjunctive)
	answer set program~$P'$
	that uses no optimization statements at all, and whose answer sets correspond
	exactly to the (lexicographically) optimized answer sets of our program~$P$.
	Since the new program~$P'$ does not use optimization,
	we can add additional constraint to~$P'$ to remove some of the answer sets.
	In particular, we will filter out those answer sets that correspond to an
	unsuccessful manipulation---i.e., those containing the
	statement~\inlinecode{unsuccessful}.
	Effectively, we add the following constraint to~$P'$:%
	\ContinueLineNumber
	\begin{lstlisting}
	:- unsuccessful. |\label{line:final-constraint}|
	\end{lstlisting}
	As a result the only answer sets of~$P'$ that remain correspond exactly
	to successful manipulations~$J'_i$ for agent~$i$.
	
	The meta-programming technique that we use uses the full disjunctive answer set programming
	language.
	For this full language, finding answer sets is a \SigmaP{2}-complete problem
	\cite{EiterGottlob95}.
	This is in line with our result of Proposition~\ref{prop:manipulation-sigmap2} where
	we show that the problem of strategic manipulation is in \SigmaP{2}.
	
	The encoding that we described can straightforwardly be modified for various variants
	of strategic manipulation (e.g., antipodal manipulation).
	To make this work, one needs to express additional constraints on the choice of
	the judgment set~$J'_i$.
	To adapt the encoding for other preference relations~$\setsucc$, one needs to adapt
	the definition of~\inlinecode{successful/0}, expressing under what conditions
	an act of manipulation is successful.
	
	Our encoding using meta-programming is relatively easily understandable, since we do not need to tinker with the encoding of complex optimization constraints in
	full disjunctive answer set programming ourselves---this we outsource to the meta-programming
	method.
	If one were to do this manually, there is more space for tailor-made optimizations,
	which might lead to a better performance of ASP solving algorithms for the problem of
	strategic manipulation.
	It is an interesting topic for future research to investigate this, and possibly to
	experimentally test the performance of different encodings, when combined with ASP solving algorithms.
	
	\section{Conclusion}\label{sec:concl}
	We have introduced the concept of egalitarianism into the framework of judgment aggregation and have presented how egalitarian and strategyproofness axioms interact in this setting. Importantly, we have shown that the two main interpretations of egalitarianism give rise to rules with differing levels of protection against manipulation. In addition, we have looked into various computational aspects of the egalitarian rules that arise from our axioms, in a twofold manner: First, we  have provided worst-case complexity results; second, we have shown how to solve the relevant hard problems using Answer Set Programming.
	
	While we have axiomatised two prominent egalitarian principles, it remains to be seen whether other egalitarian axioms can provide stronger barriers against manipulation. For example, in parallel to  majoritarian rules, one could define rules that minimise the distance to some egalitarian ideal. 
	Moreover, as is the case in judgment aggregation, there is an obvious lack of voting rules designed with egalitarian principles in mind. We hope this paper opens the door for similar explorations in voting theory.

	\balance
	
	\appendix
	\section{Appendix}
	
	In this appendix, we will show that the outcome determination problem
	for the MaxEq rule boils down to a \ThetaP{2}-complete problem.
	In particular, we will show that the outcome determination problem for the MaxEq
	rule, when seen as a search problem, is complete (under polynomial-time Levin reductions)
	for~$\FPNPlogwit{}$---which is a search variant of \ThetaP{2}.
	Then, we show that the problem of finding the minimum difference between two
	agents' satisfaction, and deciding if this value is divisible by~4, is
	a \ThetaP{2}-complete problem.
	Along the way, we show that deciding if there is a judgment set~$J^{*} \in \mathcal{J}(\Phi)$
	that has the same Hamming distance to each judgment set in a given profile~$\Pf$
	is \NP{}-hard, even in the case where the agenda consists of logically independent issues.
	
	We begin, in Section~\ref{sec:appendix-preliminaries},
	by recalling some notions from computational complexity theory---in particular,
	notions related to search problems.
	Then, in Section~\ref{sec:appendix-proofs}, we establish the computational complexity
	results mentioned above.
	
	\subsection{Additional complexity-theoretic preliminaries}
	\label{sec:appendix-preliminaries}
	
	We will consider search problems.
	Let~$\Sigma$ be an alphabet.
	A \emph{search problem} is a binary relation~$R$
	over strings in~$\Sigma^*$.
	For any input string~$x \in \Sigma^{*}$,
	we let~$R(x) = \{\  y \in \Sigma^{*} \ |\  (x,y) \in R \ \}$
	denote the set of \emph{solutions} for~$x$.
	We say that a Turing machine~$T$ \emph{solves}~$R$
	if on input~$x \in \Sigma^*$ the following holds:
	if there exists at least one~$y$ such that~$(x,y) \in R$,
	then~$T$ accepts~$x$ and outputs some~$y$ such
	that~$(x,y) \in R$;
	otherwise,~$T$ rejects~$x$.
	With any search problem~$R$ we associate a decision
	problem~$S_R$, defined by~$S_R = \{\  x \in \Sigma^{*} \ |\ 
	\text{there exists some } y \in \Sigma^{*} \text{ such that } (x,y) \in R \ \}$.
	We will use the following notion of reductions for search problems.
	A \emph{polynomial-time Levin reduction}
	from one search problem~$R_1$
	to another search problem~$R_2$ is a pair of polynomial-time
	computable functions $(g_1,g_2)$ such that:
	\begin{itemize}
		\item the function~$g_1$ is a many-one reduction from~$S_{R_1}$
		to~$S_{R_2}$, i.e., for every~$x \in \Sigma^{*}$ it holds
		that~$x \in S_{R_1}$ if and only if~$g_1(x) \in S_{R_2}$.
		\item for every string~$x \in S_{R_1}$ and every
		solution~$y \in R_2(g_1(x))$
		it holds that~$(x,g_2(x,y)) \in R_1$.
	\end{itemize}
	One could also consider other types of reductions
	for search problems,
	such as Cook reductions
	(an algorithm that solves~$R_1$ by making one or more queries to
	an oracle that solves the search problem~$R_2$).
	For more details, we refer to textbooks on the topic---e.g.,~\cite{Goldreich10}.
	
	we will use complexity classes that are based on
	Turing machines that have access to an oracle.
	Let~$C$ be a complexity class with decision problems.
	A Turing machine~$T$ with access to a \emph{yes-no~$C$ oracle}
	is a Turing machine with a dedicated \emph{oracle tape}
	and dedicated states~$q_{\text{oracle}}$,~$q_{\text{yes}}$
	and~$q_{\text{no}}$.
	Whenever~$T$ is in the state~$q_{\text{oracle}}$, it does not proceed
	according to the transition relation, but instead it transitions into
	the state~$q_{\text{yes}}$ if the oracle tape contains
	a string~$x$ that is a yes-instance for the problem~$C$, i.e.,~if~$x \in C$,
	and it transitions into the state~$q_{\text{no}}$ if~$x \not\in C$.
	Let~$C$ be a complexity class with search problems.
	Similarly, a Turing machine with access to a
	\emph{witness~$C$ oracle} has a dedicated oracle tape
	and dedicated states~$q_{\text{oracle}}$,~$q_{\text{yes}}$
	and~$q_{\text{no}}$.
	Also, whenever~$T$ is in the state~$q_{\text{oracle}}$
	it transitions into
	the state~$q_{\text{yes}}$ if the oracle tape contains
	a string~$x$ such that there exists some~$y$ such that~$C(x,y)$,
	and in addition the contents of the oracle tape are replaced
	by (the encoding of) such an~$y$;
	it transitions into the state~$q_{\text{no}}$ if there exists no~$y$ such
	that~$C(x,y)$.
	Such transitions are called \emph{oracle queries}.
	
	We consider the following complexity classes
	that are based on oracle machines.
	\begin{itemize}
		\item The class~\PNPlog{} consists of all decision problems
		that can be decided by a deterministic polynomial-time Turing machine that
		has access to a yes-no \NP{} oracle, and on any input of length~$n$
		queries the oracle at most~$O(\log n)$ many times.
		This class coincides with the class \PNPpar{} (spoken: ``parallel access to NP''),
		and is also known as \ThetaP{2}.
		
		Incidentally, allowing the algorithms access to a witness \FNP{} oracle
		instead of access to a yes-no \NP{} oracle leads to the same class
		of problems, i.e., the class \PNPlogwit{} that
		coincides with \PNPlog{} (cf.~\cite[Corollary~6.3.5]{Krajicek95}).
		
		\item The class~\FPNPlogwit{} consists of all search problems
		that can be solved by a deterministic polynomial-time Turing machine that
		has access to a witness \FNP{} oracle, and on any input of
		length~$n$ queries the oracle at most~$O(\log n)$ many times.
		In a sense, it is the search variant of \PNPlog{}.
		
		This complexity class happens to coincide with the class
		\FNPOptPlog{}, which is defined as the set of all search problems
		that are solvable by a nondeterministic polynomial-time Turing machine
		that receives as advice the answer to one ``NP optimization'' computation
		\cite{ChenToda95,KoeblerThierauf90}.
	\end{itemize}
	
	\subsection{Complexity proofs for the MaxEq rule}
	\label{sec:appendix-proofs}
	
	We define the search problem of outcome determination for a judgment aggregation
	rule~$F$ as follows.
	The input for this problem consists of an agenda~$\Phi$,
	a profile~$\Pf = (J_1,\dotsc,J_n) \in \mathcal{J}(\Phi)^{n}$.
	The problem is to output some judgment set~$J^{*} \in F(\Pf)$.
	In other words, the problem is the relation~$R$ that consists of all
	pairs~$((\Phi,\Pf),J^{*})$ such that~$\Pf \in \mathcal{J}(\Phi)^{n}$ is a
	profile and~$J^{*} \in F(\Pf)$.
	We will show that the outcome determination problem for the MaxEq rule
	is complete for the complexity class~$\FPNPlogwit{}$
	under polynomial-time Levin reductions.
	
	In order to do so, we begin with establishing a lemma that will be
	useful for the $\FPNPlogwit{}$-hardness proof.
	This lemma uses the notion of 1-in-3-satisfiability.
	Let~$\psi$ be a propositional logic formula in 3CNF,
	i.e.,~$\psi = c_1 \wedge \dotsm \wedge c_m$,
	where each~$c_i$ is a clause containing exactly three literals.
	Then~$\psi$ is \emph{1-in-3-satisfiable} if there exists
	a truth assignment~$\alpha$ that satisfies exactly one of the three literals
	in each clause~$c_i$.
	
	\begin{lemma}
		\label{lem:lemma-maxeq-hardness}
		Let~$\psi$ be a 3CNF formula with clauses~$c_1,\dotsc,c_b$ that are all of size exactly~3
		and with~$n$ variables~$x_1,\dotsc,x_n$,
		such that
		(1)~no clause of~$\psi$ contains complementary literals,
		and (2)~there exists some~$x^{*} \in \{ x_1,\dotsc,x_n \}$ and
		a partial truth assignment~$\beta : \{ x_1,\dotsc,x_n \} \setminus \{ x^{*} \} \rightarrow \{0,1\}$
		that satisfies exactly one literal in each clause where~$x^{*}$ or~$\neg x^{*}$ occurs,
		and satisfies no literal in each clause where~$x^{*}$ or~$\neg x^{*}$ occurs.
		We can, in polynomial time given~$\psi$, construct an agenda~$\Phi$ on~$m$ issues
		such that~$\mathcal{J}(\Phi) = \{ 0,1 \}^{m}$ and a profile~$\Pf$ over~$\Phi$, such that:
		\begin{itemize}
			\item $\Phi = \{\ y_i, y'_i\ |\ 1 \leq i \leq n\ \} \cup \{ z_1,\dotsc,z_5 \}$;
			\item there exists a judgment set~$J \in \mathcal{J}(\Phi)$
			that has the same Hamming distance to each~$J' \in \Pf$
			if and only if~$\psi$ is 1-in-3-satisfiable;
			\item if~$\psi$ is not 1-in-3-satisfiable, then for each judgment
			set~$J \in \mathcal{J}(\Phi)$ it holds
			that~$\max_{J',J'' \in \Pf} |H(J,J') - H(J,J'')| \geq 2$,
			and there exists some~$J \in \mathcal{J}(\Phi)$
			such that~$\max_{J',J'' \in \Pf} |H(J,J') - H(J,J'')| = 2$;
			\item the above two properties hold also when restricted to judgment
			sets~$J$ that contain exactly one of~$y_i$ and~$y'_i$
			for each~$1 \leq i \leq n$, and that contain~$\neg z_1,\dotsc,\neg z_4, z_5$; and
			\item the number of judgment sets in the profile~$\Pf$ only depends
			on~$n$ and~$b$.
		\end{itemize}
	\end{lemma}
	\begin{proof}
		We let the agenda~$\Phi$ consist of
		the~$2n+5$ issues~$y_1,\dotsc,y_n$, $y'_1,\dotsc,y'_n,z_1,\dotsc,z_5$.
		It follows directly that~$\mathcal{J}(\Phi) = \{0,1\}^{2n+5}$.
		Then, we start by constructing~$2n$ judgment sets~$J_1,\dotsc,J_{2n}$ over these issues,
		defined as depicted in Figure~\ref{fig:lemma-maxeq-hardness-1}.
		
		\begin{figure}[h!]
			\begin{tabular}{c | cccc | cccc | ccc}
				\toprule
				& $y_1$ & $y_2$ & $\dotsm$ & $y_n$ & $y'_1$ & $y'_2$ & $\dotsm$ & $y'_n$ & $z_1$ & $\dotsm$ & $z_5$ \\
				\midrule \midrule
				$J_1$ & 1&0&$\dotsm$&0 & 1&0&$\dotsm$&0 & 0&$\dotsm$&0 \\
				$J_2$ & 0&1&$\dotsm$&0 & 0&1&$\dotsm$&0 & 0&$\dotsm$&0 \\
				$\vdots$ & $\vdots$&$\vdots$&$\ddots$&$\vdots$ & $\vdots$&$\vdots$&$\ddots$&$\vdots$ & $\vdots$&$\ddots$&$\vdots$\\
				$J_n$ & 0&0&$\dotsm$&1 & 0&0&$\dotsm$&1 & 0&$\dotsm$&0 \\
				\midrule
				$J_{n+1}$ & 0&1&$\dotsm$&1 & 0&1&$\dotsm$&1 & 0&$\dotsm$&0 \\
				$J_{n+2}$ & 1&0&$\dotsm$&1 & 1&0&$\dotsm$&1 & 0&$\dotsm$&0 \\
				$\vdots$ & $\vdots$&$\vdots$&$\ddots$&$\vdots$ & $\vdots$&$\vdots$&$\ddots$&$\vdots$ & $\vdots$&$\ddots$&$\vdots$\\
				$J_{2n}$ & 1&1&$\dotsm$&0 & 1&1&$\dotsm$&0 & 0&$\dotsm$&0 \\
				\bottomrule
			\end{tabular}
			\caption{Construction of the judgment sets~$J_1,\dotsc,J_{2n}$ in the proof
				of Lemma~\ref{lem:lemma-maxeq-hardness}.}
			\label{fig:lemma-maxeq-hardness-1}
		\end{figure}
		
		Then, for each clause~$c_k$ of~$\psi$, we introduce three judgment sets~$J_{k,1}$, $J_{k,2}$,
		and~$J_{k,3}$
		that are defined as follows.
		The judgment set~$J_{k,1}$ contains~$y_i, \neg y'_i$
		for each positive literal~$x_i$ occurring in~$c_k$,
		and contains~$y'_i, \neg y_i$ for each negative literal~$\neg x_i$ occurring in~$c_k$.
		Conversely, the judgment sets~$J_{k,2},J_{k,3}$ contain~$y'_i, \neg y_i$
		for each positive literal~$x_i$ occurring in~$c_k$,
		and contain~$y_i, \neg y'_i$ for each negative literal~$\neg x_i$ occurring in~$c_k$.
		For each variable~$x_j$ that does not occur in~$c_k$,
		all three of~$J_{k,1},J_{k,2},J_{k,3}$ contain~$\neg y_j, \neg y'_j$.
		Finally, the judgment set~$J_{k,1}$ contains~$\neg z_1,\dotsc,\neg z_4,z_5$,
		the judgment set~$J_{k,2}$ contains~$\neg z_1,\neg z_2,z_3,z_4,z_5$,
		and the judgment set~$J_{k,3}$ contains~$z_1,z_2,\neg z_3,\neg z_4,z_5$.
		This is illustrated
		in Figure~\ref{fig:lemma-maxeq-hardness-2}
		for the example clause~$(x_1 \vee \neg x_2 \vee x_3)$.
		
		\begin{figure}[h!]
			\begin{tabular}{c | ccc@{\ \ }c | ccc@{\ \ }c | c@{\ \ }c@{\ \ }c@{\ \ }c@{\ \ }c}
				\toprule
				& $y_1$ & $y_2$ & $y_3$ & $\dotsm$ & $y'_1$ & $y'_2$ & $y'_3$ & $\dotsm$ & $z_1$ & $z_2$ & $z_3$ & $z_4$ & $z_5$ \\
				\midrule \midrule
				$J_{k,1}$ & 1&0&1&$\dotsm$&0&1&0&$\dotsm$ & 0&0&0&0&1 \\
				$J_{k,2}$ & 0&1&0&$\dotsm$&1&0&1&$\dotsm$ & 0&0&1&1&1 \\
				$J_{k,3}$ & 0&1&0&$\dotsm$&1&0&1&$\dotsm$ & 1&1&0&0&1 \\
				\bottomrule
			\end{tabular}
			\caption{Illustration of the construction of the
				judgment sets~$J_{k,1},J_{k,2},J_{k,3}$ for the example
				clause~$c_k = (x_1 \vee \neg x_2 \vee x_3)$ in the proof
				of Lemma~\ref{lem:lemma-maxeq-hardness}.}
			\label{fig:lemma-maxeq-hardness-2}
		\end{figure}
		
		The profile~$\Pf$ then consists of the judgment sets~$J_1,\dotsc,J_{2n}$,
		as well as the judgment sets~$J_{k,1},J_{k,2},J_{k,3}$ for each~$1 \leq k \leq m$.
		In order to prove that~$\Pf$ has the required properties,
		we consider the following observations and claims
		(and prove the claims).
		
		\begin{obs}
			\label{obs:lemma-maxeq-hardness-1}
			If a judgment set~$J$ contains exactly one of~$y_i$ and~$y'_i$ for each~$i$,
			then the Hamming distance from~$J$ to each of~$J_1,\dotsc,J_{2n}$---restricted to the
			issues~$y_1,\dotsc,y_n,y'_1,\dotsc,y'_n$---is exactly~$n$.
		\end{obs}
		
		\begin{cla}
			\label{claim:lemma-maxeq-hardness-2}
			If a judgment set~$J$ contains both~$y_i$ and~$y'_i$ for some~$i$,
			or neither~$y_i$ nor~$y'_i$ for some~$i$,
			then there are at least two judgment sets among~$J_1,\dotsc,J_{2n}$ such that
			the Hamming distance from~$J$ to these two judgment sets differs by at least~$2$.
		\end{cla}
		\begin{claimproof}[Proof of Claim~\ref{claim:lemma-maxeq-hardness-2}]
			We argue that this is the case for~$y_1$ and~$y'_1$, i.e., for the case of~$i = 1$.
			For other values of~$i$, an entirely similar argument works.
			
			Picking both~$y_1$ and~$y'_1$ to be part of a judgment set~$J$ adds to the Hamming distances
			between~$J$, on the one hand, and~$J_1,\dotsc,J_{2n}$, on the other hand,
			according to the vector~$v^{+}_{1}$:
			\[ v^{+}_{1} = (0,\underbrace{2,\dotsc,2}_{n-1},2,\underbrace{0,\dotsc,0}_{n-1}). \]
			Picking both~$\neg y_1$ and~$\neg y'_1$ to be part of the set~$J$ adds to the Hamming distances
			between~$J$ and~$J_1,\dotsc,J_{2n}$ according to the vector~$v_{\overline{1}}$:
			\[ v^{-}_{1} = (2,\underbrace{0,\dotsc,0}_{n-1},0,\underbrace{2,\dotsc,2}_{n-1}). \]
			Picking exactly one of~$y_1$ and~$y'_1$ and exactly one of~$\neg y_1$ and~$\neg y'_1$
			to be part of~$J$ corresponds to the all-ones vector~$\bm{1} = (1,\dotsc,1)$.
			More generally, both~$y_i$ and~$y'_i$ to be part of~$J$ adds to the Hamming distances
			according to the vector~$v^{+}_{i}$:
			\[ v^{+}_{i} = (\underbrace{2,\dotsc,2}_{i-1},0,\underbrace{2,\dotsc,2}_{n-i+2},\underbrace{0,\dotsc,0}_{i-1},2,\underbrace{0,\dotsc,0}_{n-i+2}), \]
			picking both~$\neg y_i$ and~$\neg y'_i$ corresponds to the vector~$v^{-}_{i}$:
			\[ v^{-}_{i} = (\underbrace{0,\dotsc,0}_{i-1},2,\underbrace{0,\dotsc,0}_{n-i+2},\underbrace{2,\dotsc,2}_{i-1},0,\underbrace{2,\dotsc,2}_{n-i+2}), \]
			and picking exactly one of~$y_i$ and~$y'_i$ and exactly one of~$\neg y_i$ and~$\neg y'_i$
			to be part of~$J$ corresponds to the all-ones vector~$\bm{1} = (1,\dotsc,1)$.
			For each~$1 \leq i \leq n$,
			the vectors~$v^{-}_{1},\dotsc,v^{-}_{n}$ and~$v^{+}_{i}$ are linearly independent,
			and the vectors~$v^{+}_{1},\dotsc,v^{+}_{n}$ and~$v^{-}_{i}$ are linearly independent.
			
			Suppose now that we pick~$J$ to contain both~$y_1$ and~$y'_1$.
			Suppose, moreover, to derive a contradiction, that the Hamming distance from~$J$
			to each of the judgment sets~$J_1,\dotsc,J_{2n}$ is the same.
			This means that there is some way of choosing~$s_2,\dotsc,s_n$
			such that~$v^{-}_{1} = \sum_{1 < i \leq n} v^{s_i}_{i}$,
			which contradicts the fact that~$v^{+}_{1},\dotsc,v^{+}_{n}$
			and~$v^{-}_{i}$ are linearly independent---since each~$v^{-}_{j}$
			can be expressed as~$v^{-}_{j} = v^{+}_{i} + v^{-}_{i} - v^{+}_{j}$.
			Thus, we can conclude that there exist at least two judgment sets
			among~$J_1,\dotsc,J_{2n}$ such that
			the Hamming distance from~$J$ to these two judgment sets differs.
			Moreover, since all vectors contain only even numbers, and the coefficients
			in the sum are integers (in fact, either 0 or 1),
			we know that the difference must be even and thus at least 2.
			
			An entirely similar argument works for the case where we pick~$J$
			to contain both~$\neg y_1$ and~$\neg y'_1$.
		\end{claimproof}
		
		\begin{obs}
			\label{obs:lemma-maxeq-hardness-3}
			Let~$\alpha : \{ x_1,\dotsc,x_n \} \rightarrow \{0,1\}$ be a truth
			assignment that satisfies exactly one literal in each clause of~$\psi$.
			Consider the judgment set~$J_{\alpha} = \{\ y_i, \neg y'_i\ |\ 1 \leq i \leq n, \alpha(x_i) = 1\ \} \cup
			\{\ y'_i, \neg y_i\ |\ 1 \leq i \leq n, \alpha(x_i) = 0\ \} \cup \{ \neg z_1,\dotsc,\neg z_4,z_5 \}$.
			Then the Hamming distance from~$J_{\alpha}$ to each judgment set
			in the profile~$\Pf$ is exactly~$n+1$.
		\end{obs}
		
		\begin{cla}
			\label{claim:lemma-maxeq-hardness-4}
			Suppose that~$\psi$ is not 1-in-3-satisfiable.
			Then for each judgment set~$J$ that contains exactly one of~$y_i$ and~$y'_i$
			for each~$i$, there is some clause~$c_k$ of~$\psi$ such that
			the difference in Hamming distance from~$J$ to (two of)~$J_{k,1},J_{k,2},J_{k,3}$
			is at least~$2$.
		\end{cla}
		\begin{claimproof}[Proof of Claim~\ref{claim:lemma-maxeq-hardness-4}]
			Take an arbitrary judgment set~$J$ that contains exactly one of~$y_i$ and~$y'_i$
			for each~$i$.
			This judgment set~$J$ corresponds to the truth
			assignment~$\alpha_{J} : \{ x_1,\dotsc,x_n \} \rightarrow \{0,1\}$ defined
			such that for each~$1 \leq i \leq n$ it is the case that~$\alpha(x_i) = 1$ if~$y_i \in J$
			and~$\alpha(x_i) = 0$ if~$y_i \not\in J$.
			Since~$\psi$ is not 1-in-3-satisfiable, we know that there exists some clause~$c_{\ell}$
			such that~$\alpha_{J}$ does not satisfy exactly one literal in~$c_{\ell}$.
			We distinguish several cases:
			either~(i)~$\alpha_{J}$ satisfies no literals in~$c_{\ell}$,
			or~(ii)~$\alpha_{J}$ satisfies two literals in~$c_{\ell}$,
			or~(iii)~$\alpha_{J}$ satisfies three literals in~$c_{\ell}$.
			In each case, the Hamming distances from~$J$, on the one hand,
			and~$J_{k,1},J_{k,2},J_{k,3}$, on the other hand, must differ by at least~$2$.
			This can be verified case by case---and we omit a further detailed case-by-case
			verification of this.
		\end{claimproof}
		
		\begin{cla}
			\label{claim:lemma-maxeq-hardness-5}
			If~$\psi$ is not 1-in-3-satisfiable, then there exists a judgment set~$J$
			such that~$\max_{J',J'' \in \Pf} |H(J,J') - H(J,J'')| = 2$.
		\end{cla}
		\begin{claimproof}[Proof of Claim~\ref{claim:lemma-maxeq-hardness-5}]
			Suppose that~$\psi$ is not 1-in-3-satisfiable.
			We know that there exists a variable~$x^{*} \in \{ x_1,\dotsc,x_n \}$ and
			a partial truth assignment~$\beta : \{ x_1,\dotsc,x_n \} \setminus \{ x^{*} \} \rightarrow \{0,1\}$
			that satisfies exactly one literal in each clause where~$x^{*}$ or~$\neg x^{*}$ does not occur,
			and satisfies no literal in each clause where~$x^{*}$ or~$\neg x^{*}$ occurs.
			Without loss of generality, suppose that~$x^{*} = x_1$.
			Now consider the judgment set~$J_{\beta} = \{ \neg y_1, \neg y'_1 \} \cup
			\{\ y_i, \neg y'_i\ |\ 1 < i \leq n, \beta(x_i) = 1\ \} \cup
			\{\ y'_i, \neg y_i\ |\ 1 < i \leq n, \beta(x_i) = 0\ \} \cup \{ \neg z_1,\dotsc,\neg z_4,z_5 \}$.
			One can verify that the Hamming distances from~$J_{\beta}$, on the one hand,
			and~$J' \in \Pf$, on the other hand, differ by at most~2---and for some~$J',J'' \in \Pf$
			it holds that~$|H(J,J') - H(J,J'')| = 2$.
		\end{claimproof}
		
		We now use the above observations and claims to show that~$\Phi$ and~$\Pf$
		have the required properties.
		If~$\psi$ is 1-in-3-satisfiable, by Observation~\ref{obs:lemma-maxeq-hardness-3},
		there is some~$J \in \mathcal{J}(\Phi)$ that has the same Hamming distance
		to each~$J' \in \Pf$.
		Suppose, conversely, that~$\psi$ is not 1-in-3-satisfiable.
		Then by Claims~\ref{claim:lemma-maxeq-hardness-2} and~\ref{claim:lemma-maxeq-hardness-4},
		there exist two judgment sets~$J',J'' \in \Pf$ such that~$H(J,J')$ and~$H(J,J'')$
		differ (by at least 2).
		Thus,~$\psi$ is 1-in-3-satisfiable if and only if there exists a judgment
		set~$J \in \mathcal{J}(\Phi)$ that has the same Hamming distance to each
		judgment set in the profile~$\Pf$.
		
		Suppose that~$\psi$ is not 1-in-3-satisfiable.
		Then by Claims~\ref{claim:lemma-maxeq-hardness-2}
		and~\ref{claim:lemma-maxeq-hardness-4}, we know that for
		each~$J \in \mathcal{J}(\Phi)$ it holds
		that $\max_{J',J'' \in \Pf} |H(J,J') - H(J,J'')| \geq 2$.
		Moreover, by Claim~\ref{claim:lemma-maxeq-hardness-5},
		there exists a judgment set~$J \in \mathcal{J}(\Phi)$
		such that~$\max_{J',J'' \in \Pf} |H(J,J') - H(J,J'')| = 2$.
		
		We already observed that that~$\mathcal{J}(\Phi) = \{0,1\}^{m}$
		for some~$m$.
		Moreover, one can straightforwardly verify that
		the statements after the first two bullet points in the statement
		of the lemma also hold when restricted to judgment
		sets~$J$ that contain exactly one of~$y_i$ and~$y'_i$ for each~$1 \leq i \leq n$
		and that contain~$\neg z_1,\dotsc,\neg z_4, z_5$.
		This concludes the proof of the lemma.
	\end{proof}
	
	Now that we have established the lemma,
	we continue with the $\FPNPlogwit{}$-completeness proof.
	
	\duplicatethm{3}%
	{The problem of outcome determination for the MaxEq rule
		is~$\FPNPlogwit{}$-complete under polynomial-time Levin reductions.}
	\begin{proof}
		To show membership in \FPNPlogwit{}, we describe an algorithm
		with access to a witness \FNP{} oracle that solves the problem in polynomial time
		by making at most a logarithmic number of oracle queries.
		This algorithm will use an oracle for the following \FNP{} problem:
		given some~$k \in \mathbb{N}$ and given the agenda~$\Phi$ and
		the profile~$\Pf$, compute a judgment set~$J \in \mathcal{J}(\Phi)$
		such that~$\max_{J',J'' \in \Pf} |H(J,J') - H(J,J'')| \leq k$, if such a~$J$ exists,
		and return ``none'' otherwise.
		By using~$O(\log |\Phi|)$ queries to this oracle, one can compute
		the minimum value~$k_{\text{min}}$ of~$\max_{J',J'' \in \Pf} |H(J,J') - H(J,J'')|$
		where the minimum is taken over all judgment sets~$J \in \mathcal{J}(\Phi)$.
		Then, with a final query to the oracle, using~$k = k_{\text{min}}$,
		one can use the oracle to produce a judgment set~$J^{*} \in \text{MaxEq}(\Pf)$.
		
		We will show $\FPNPlogwit{}$-hardness
		by giving a polynomial-time Levin reduction
		from the $\FPNPlogwit{}$-complete problem of finding
		a satisfying assignment for a (satisfiable)
		propositional formula~$\psi$ that sets a
		maximum number of variables to true
		(among any satisfying assignment of~$\psi$)
		\cite{ChenToda95,KoeblerThierauf90,Krentel88,Wagner90}.
		Let~$\psi$ be an arbitrary satisfiable propositional logic formula with~$v$ variables.
		Without loss of generality, assume that~$v$ is even
		and that there is a satisfying truth assignment for~$\psi$ that sets at least one
		variable to true.
		We will construct an agenda~$\Phi$ and a profile~$\Pf$,
		such that
		the minimum value of~$\max_{J',J'' \in \Pf} |H(J^{*},J') - H(J^{*},J'')|$,
		for any~$J^{*} \in \mathcal{J}(\Phi)$, is divisible by~4
		if and only if the maximum number of variables set to true in any satisfying assignment
		of~$\psi$ is odd.
		Moreover, we will construct~$\Phi$ and~$\Pf$ in such a way that
		from any~$J^{*} \in \text{MaxEq}(\Phi)$ we can construct,
		in polynomial time, a satisfying
		assignment of~$\psi$ that sets a maximum number of variables to true.
		This---together with the fact that~$\Phi$ and~$\Pf$ can be constructed in polynomial time---%
		suffices to exhibit a polynomial-time Levin reduction,
		and thus to show $\FPNPlogwit{}$-hardness.
		We proceed in several stages (i--iv).
		
		\begin{itemize}
			\item[(i)] We begin, in the first stage, by
			constructing a 3CNF formula~$\psi_i$ with certain properties, for each~$1 \leq i \leq v$.
		\end{itemize}
		
		\begin{cla}
			\label{claim:maxeq-hardness-1}
			We can construct in polynomial time, for each~$1 \leq i \leq v$, a 3CNF formula~$\psi_i$,
			that is 1-in-3-satisfiable
			if and only if there is a truth assignment that satisfies~$\psi$ and that sets at least~$i$
			variables in~$\psi$ to true.
			Moreover, we construct these formulas~$\psi_i$ in such a way
			that they all contain exactly the same variables~$x_1,\dotsc,x_n$
			and exactly the same number~$b$ of clauses,
			and such that each formula~$\psi_i$ has the properties:
			\begin{itemize}
				\item that it contains no clause with complementary literals, and
				\item that it contains a variable~$x^{*}$
				and a partial truth assignment~$\beta : \{ x_1,\dotsc,x_n \} \setminus \{ x^{*} \} \rightarrow \{0,1\}$
				that satisfies exactly one literal in each clause where~$x^{*}$ or~$\neg x^{*}$ does not occur,
				and satisfies no literal in each clause where~$x^{*}$ or~$\neg x^{*}$ occurs.
			\end{itemize}
			
		\end{cla}
		\begin{claimproof}[Proof of Claim~\ref{claim:maxeq-hardness-1}]
			Consider the problems of deciding if a given truth assignment~$\alpha$ to the variables
			in~$\psi$ satisfies~$\psi$, and deciding if~$\alpha$ satisfies at least~$i$ variables,
			for some~$1 \leq i \leq v$.
			These problems are both polynomial-time solvable.
			Therefore, by using standard techniques from the proof of the Cook-Levin Theorem
			\cite{Cook71}, we can construct in polynomial time a propositional formula~$\chi$ in 3CNF
			containing (among others) the variables~$t_1,\dotsc,t_v$, the variable~$x^{\dagger}$
			and the variables in~$\psi$
			such that any truth assignment to the variables~$x^{\dagger},t_1,\dotsc,t_v$ and the
			variables in~$\psi$ can be extended to a satisfying truth assignment for~$\chi$
			if and only if either~(i)~it sets~$x^{\dagger}$ to true,
			or~(ii)~it satisfies~$\psi$ and for each~$1 \leq i \leq v$ it sets~$t_i$ to false
			if and only if it sets at least~$i$ variables among the variables in~$\psi$ to true.
			Then, we can transform this 3CNF formula~$\chi$ to another 3CNF formula~$\chi'$
			with a similar property---namely that any truth assignment to the
			variables~$x^{\dagger},t_1,\dotsc,t_v$ and the
			variables in~$\psi$ can be extended to a truth assignment that satisfies exactly
			one literal in each clause of~$\chi$
			if and only if either~(i)~it sets~$x^{\dagger}$ to true,
			or~(ii)~it satisfies~$\psi$ and for each~$1 \leq i \leq v$ it sets~$t_i$ to false
			if and only if it sets at least~$i$ variables among the variables in~$\psi$ to true.
			We do so by using the the polynomial-time reduction
			from 3SAT to 1-IN-3-SAT given by \citet{Schaefer78}.
			
			Then, for each particular value of~$i$,
			we add two clauses that intuitively serve to ensure that
			variable~$t_i$ must be set to false in any 1-in-3-satisfying truth assignment.
			We add~$(s_0 \vee s_1 \vee t_i)$ and~$(\neg s_0 \vee \neg s_1 \vee t_i)$,
			where~$s_0$ and~$s_1$ are fresh variables---%
			the only way to satisfy exactly one literal in both of these clauses
			is to set exactly one of~$s_0$ and~$s_1$ to true,
			and to set~$t_i$ to false.
			
			Moreover, we add the clauses~$(r_0 \vee r_1 \vee x^{\dagger})$,
			$(\neg r_0 \vee r_2 \vee \neg x^{\dagger})$,
			$(x^{*} \vee r_3 \vee r_0)$,
			and~$(\neg x^{*} \vee r_4 \vee r_0)$,
			where~$r_0,\dotsc,r_4$ and~$x^{*}$ are fresh variables.
			These clauses serve to ensure the property that there always exists
			a partial truth assignment~$\beta : \{ x_1,\dotsc,x_n \} \setminus \{ x^{*} \} \rightarrow \{0,1\}$
			that satisfies exactly one literal in each clause where~$x^{*}$ or~$\neg x^{*}$ does not occur,
			and satisfies no literal in each clause where~$x^{*}$ or~$\neg x^{*}$ occurs.
			Moreover, these added clauses preserve 1-in-3-satisfiability if and only if
			there exists a 1-in-3-satisfying truth assignment for the formula without these added clauses
			that sets~$x^{\dagger}$ to true.
			
			Putting all this together, we have constructed a 3CNF formula~$\chi''$---consisting of~$\chi'$
			with the addition of the six clauses mentioned in the above two paragraphs---%
			that has the right properties mentioned in the statement of the claim.
			In particular,~$\chi''$ is 1-in-3-satisfiable if and only if there is a truth assignment that
			satisfies~$\psi$ and that sets at least~$i$ variables in~$\psi$ to true.
			Moreover, the constructed formula~$\chi''$ has the same variables
			and the same number of clauses, regardless of the value of~$i$ chosen in the construction.
		\end{claimproof}
		
		Since the formulas~$\psi_i$, as described in Claim~\ref{claim:maxeq-hardness-1},
		satisfy the requirements for Lemma~\ref{lem:lemma-maxeq-hardness},
		we can construct agendas~$\Phi_1,\dotsc,\Phi_v$
		and profiles~$\Pf_1,\dotsc,\Pf_v$ such that for each~$1 \leq i \leq v$,
		the agenda~$\Phi_i$ and the profile~$\Pf_i$ satisfy the conditions mentioned
		in the statement of Lemma~\ref{lem:lemma-maxeq-hardness}.
		Moreover, we can construct the agendas~$\Phi_1,\dotsc,\Phi_v$ in such a way
		that they are pairwise disjoint.
		For each~$1 \leq i \leq v$,
		let~$y_{i,1},\dotsc,y_{i,n},y'_{i,1},\dotsc,y'_{i,n},
		z_{i,1},\dotsc,z_{i,5}$ denote the issues in~$\Phi_i$
		and let~$\Pf_i = (J_{i,1},\dotsc,J_{i,u})$.
		
		\begin{itemize}
			\item[(ii)] Then, in the second stage,
			we will use the profiles~$\Phi_1,\dotsc,\Phi_v$
			and the profiles~$\Pf_1,\dotsc,\Pf_v$ to construct
			a single agenda~$\Phi$ and a single profile~$\Pf$.
		\end{itemize}
		
		We let~$\Phi$ contain
		the issues~$y_{i,j}$,~$y'_{i,j}$ and~$z_{i,1},\dotsc,z_{i,5}$,
		for each~$1 \leq i \leq v$ and each~$1 \leq j \leq n$,
		as well as
		issues~$w_{i,\ell,k}$
		for each~$1 \leq i \leq v$, each~$1 \leq \ell \leq u$
		and each~$1 \leq k \leq n$.
		We let~$\Pf$ contain judgment sets~$J'_{i,\ell}$
		for each~$1 \leq i \leq v$ and each~$1 \leq \ell \leq u$,
		that we will define below.
		Intuitively, for each~$i$, the sets~$J'_{i,\ell}$
		will contain the judgment sets~$J_{i,1},\dotsc,J_{i,u}$ from the profile~$\Pf_i$.
		
		Take an arbitrary~$1 \leq i \leq v$,
		and an arbitrary~$1 \leq \ell \leq u$.
		We let~$J'_{i,\ell}$ agree with~$J_{i,\ell}$ (from~$\Pf_i$)
		on all issues from~$\Phi_i$---i.e., the issues~$y_{i,j}$,~$y'_{i,j}$ for each~$1 \leq j \leq n$
		and~$z_{i,1},\dotsc,z_{i,5}$.
		On all issues~$\varphi$ from each~$\Phi_{i'}$, for~$1 \leq i' \leq v$ with~$i' \neq i$,
		we let~$J'_{i,\ell}(\varphi) = 0$.
		Then, we let~$J'_{i,\ell}(w_{i',\ell',k}) = 1$ if and
		only if~$i = i'$ and~$\ell = \ell'$.
		In other words,~$J'_{i,\ell}$ agrees with~$J_{i,\ell}$ on the issues from~$\Phi_i$,
		it sets every issue from each other~$\Phi_{i'}$ to false,
		it sets all the issues~$w_{i,\ell,k}$ to true,
		and it sets all other issues~$w_{i',\ell',k}$ to false.
		
		\begin{itemize}
			\item[(iii)] In the third stage, we will replace the logically independent
			issues in~$\Phi$ by other issues, in order to place restrictions on the different
			judgment sets that are allowed.
		\end{itemize}
		We start by describing a constraint---in the form of a propositional logic formula~$\Gamma$
		on the original (logically independent) issues in~$\Phi$---and then we describe
		how this constraint can be used to produce replacement formulas for the issues
		in~$\Phi$.
		We define~$\Gamma = \Gamma_{1} \vee \Gamma_{2}$ by:
		\[ \begin{array}{r l}
		\Gamma_{1} = & \bigvee\limits_{1 \leq i \leq v \atop 1 \leq \ell \leq u}
		\left (
		\bigwedge\limits_{1 \leq k \leq n} w_{i,\ell,k}
		\wedge
		\bigwedge\limits_{1 \leq i' \leq v, 1 \leq \ell \leq u \atop i \neq i' \text{ or } \ell \neq \ell'} \neg w_{i',\ell',k}
		\right ) \\
		\Gamma_{2} = &
		\left (
		\bigwedge\limits_{1 \leq i \leq v, 1 \leq \ell \leq u \atop 1 \leq k \leq n} \neg w_{i,\ell,k}
		\right )
		\wedge
		\left (
		\bigwedge\limits_{1 \leq i \leq v \atop 1 \leq j \leq n} (y_{i,j} \leftrightarrow \neg y'_{i,j})
		\right ) \\
		\end{array} \]
		In other words,~$\Gamma$ requires that
		either~(1)~for some~$i,\ell$ all issues~$w_{i,\ell,k}$ are set to true,
		and all other issues~$w_{i',\ell',k}$ are set to false,
		or~(2)~all issues~$w_{i,\ell,k}$ are set to false
		and for each~$i,j$ exactly one of~$y_{i,j}$ and~$y'_{i,j}$ is set to true.
		
		Because we can in polynomial time compute a satisfying truth assignment for~$\Gamma$,
		we know that we can also in polynomial time compute a replacement~$\varphi'$
		for each~$\varphi \in \Phi$, resulting in an agenda~$\Phi'$,
		so that the logically consistent judgment sets~$J \in \mathcal{J}(\Phi')$
		correspond exactly to the judgment sets~$J \in \mathcal{J}(\Phi)$
		that satisfy the constraint~$\Gamma$
		\cite[Proposition~3]{EndrissGrandiDeHaanLang16}.
		In the remainder of this proof, we will use~$\Phi$ (with the restriction that judgment
		sets should satisfy~$\Gamma$) interchangeably with~$\Phi'$.
		
		\begin{itemize}
			\item[(iv)] Finally, in the fourth stage, we will duplicate some issues~$\varphi \in \Phi$
			a certain number of times, by adding semantically equivalent (yet syntactically different)
			issues to the agenda~$\Phi$, and by updating the judgment sets in the profile~$\Pf$
			accordingly.
		\end{itemize}
		
		For each~$1 \leq i \leq v$, we will define a number~$c_i$
		and will make sure that there are~$c_i$ (syntactically different, logically equivalent) copies
		of each issue in~$\Phi$ that originated from the agenda~$\Phi_i$.
		In other words, we duplicate each agenda~$\Phi_i$ a certain number of times,
		by adding~$c_i - 1$ copies of each issue that originated from~$\Phi_i$.
		For each~$1 \leq i \leq v$, we let~$c_i = (v-i)$.
		
		This concludes the description of our reduction---i.e.,
		of the profile~$\Phi$ and the profile~$\Pf$.
		What remains is to show that the
		minimum value of~$\max_{J',J'' \in \Pf} |H(J^{*},J') - H(J^{*},J'')|$,
		for any~$J^{*} \in \mathcal{J}(\Phi)$,
		is divisible by~4 if and only
		if the maximum number of variables set to true in any satisfying assignment
		of~$\psi$ is odd.
		To do so, we begin with stating and proving the following claims.
		
		\begin{cla}
			\label{claim:maxeq-hardness-2}
			For any judgment set~$J^{*} \in \mathcal{J}(\Phi)$ that satisfies~$\Gamma_1$,
			the value of~$\max_{J',J'' \in \Pf} |H(J^{*},J') - H(J^{*},J'')|$
			is at least~$2n$.
		\end{cla}
		\begin{claimproof}[Proof of Claim~\ref{claim:maxeq-hardness-2}]
			Take some~$J^{*}$ that satisfies~$\Gamma_{1}$.
			Then there must exist some~$i,\ell$ such that~$J$ sets all~$w_{i,\ell,k}$
			to true and all other~$w_{i',\ell',k}$ to false.
			Therefore,~$|H(J^{*},J'_{i,\ell}) - H(J^{*},J_{i',\ell'})|$ is at least~$2n$,
			for any~$i',\ell'$ such that~$(i,\ell) \neq (i',\ell')$.
		\end{claimproof}
		
		\begin{cla}
			\label{claim:maxeq-hardness-3}
			For any judgment set~$J^{*} \in \mathcal{J}(\Phi)$ that satisfies~$\Gamma_2$,
			the value of~$\max_{J',J'' \in \Pf} |H(J^{*},J') - H(J^{*},J'')|$
			is strictly less than~$2v$.
		\end{cla}
		\begin{claimproof}[Proof of Claim~\ref{claim:maxeq-hardness-3}]
			Take some~$J^{*}$ that satisfies~$\Gamma_{2}$.
			Then~$J^{*}$ sets each~$w_{i,\ell,k}$ to false,
			and for each~$1 \leq i \leq v$ the judgment set~$J$
			contains~$\neg z_{i,1},\dotsc,\neg z_{i,4},z_{i,5}$
			and contains exactly one of~$y_{i,j}$ and~$y'_{i,j}$ for each~$1 \leq j \leq n$.
			Moreover, without loss of generality, we may assume that~$J^{*}$,
			for each~$1 \leq i \leq v$, assigns truth values to the issues originating from~$\Phi_i$
			in a way that corresponds either (a)~to a truth assignment to the variables in~$\psi_i$
			witnessing that~$\psi_i$ is 1-in-3-satisfiable, or (b)~to a partial truth
			assignment~$\beta : \text{var}(\psi_i) \setminus \{ x_1 \} \rightarrow \{0,1\}$
			that satisfies exactly one literal in each clause where~$x_1$ or~$\neg x_1$ occurs,
			and satisfies no literal in each clause where~$x_1$ or~$\neg x_1$ occurs.
			If this were not the case, we could consider another~$J^{*}$ instead that does satisfy
			these properties and that has a value of~$\max_{J',J'' \in \Pf} |H(J^{*},J') - H(J^{*},J'')|$
			that is at least as small.
			
			Then, by the construction of~$\Pf$---and the profiles~$\Pf_i$ used in this construction---%
			for each~$J'_{i,j}$ it holds that $H(J^{*},J'_{i,j}) = \sum\nolimits_{1 \leq i \leq v}c_i(n+1) + n \pm d_i$,
			where~$d_i = 0$ if~$\psi_i$ is 1-in-3-satisfiable
			and~$d_i = (v-i)$ otherwise.
			From this it follows that
			the value of~$\max_{J',J'' \in \Pf} |H(J^{*},J') - H(J^{*},J'')|$
			is strictly less than~$2v$, since~$\psi_1$ is 1-in-3-satisfiable.
		\end{claimproof}
		
		By Claims~\ref{claim:maxeq-hardness-2} and~\ref{claim:maxeq-hardness-3},
		and because we may assume without loss of generality that~$n \geq v$,
		we know that any judgment set~$J^{*} \in \mathcal{J}(\Phi)$
		that minimizes~$\max_{J',J'' \in \Pf} |H(J^{*},J') - H(J^{*},J'')|$
		must satisfy~$\Gamma_2$.
		Moreover, by a straightforward modification of the arguments in the proof
		of Claim~\ref{claim:maxeq-hardness-3},
		we know that the minimal value, over judgment sets~$J^{*} \in \mathcal{J}(\Phi)$,
		of~$\max_{J',J'' \in \Pf} |H(J^{*},J') - H(J^{*},J'')|$ is~$2(v-i)$
		for the smallest value of~$i$ such that~$\psi_i$ is not 1-in-3-satisfiable,
		which coincides with~$2(v+1-i)$ for the largest value of~$i$
		such that~$\psi_i$ is 1-in-3-satisfiable.
		Since~$v$ is even (and thus~$v+1$ is odd), we know that~$2(v+1-i)$ is divisible
		by~4 if and only if~$i$ is odd.
		Therefore, the minimal value of~$\max_{J',J'' \in \Pf} |H(J^{*},J') - H(J^{*},J'')|$
		is divisible by~4 if and only if
		the maximum number of variables set to true in any satisfying assignment of~$\psi$ is odd.
		
		Moreover, it is straightforward to show that from
		any~$J^{*} \in \text{MaxEq}(\Phi)$ we can construct,
		in polynomial time, a satisfying
		assignment of~$\psi$ that sets a maximum number of variables to true.
		
		This concludes our description and analysis of the polynomial-time Levin reduction,
		and thus of our hardness proof.
	\end{proof}
	
	\duplicateprop{5}%
	{Given an agenda~$\Phi$ and a profile~$\Pf$,
		deciding whether the minimal value
		of~$\max_{J',J'' \in \Pf} |H(J^{*},J') - H(J^{*},J'')|$
		for~$J^{*} \in \mathcal{J}(\Phi)$,
		is divisible by~4,
		is a \ThetaP{2}-complete problem.}
	\begin{proof}
		This follows from the proof of Theorem~\ref{thm:fpnplogwit-completeness}.
		The proof of membership in~$\FPNPlogwit{}$ for the outcome determination
		problem for the MaxEq rule can directly be used to show
		membership in~\ThetaP{2} for this problem.
		Moreover, the polynomial-time Levin reduction used in the $\FPNPlogwit{}$-hardness proof
		can be seen as a polynomial-time (many-to-one) reduction
		from the \ThetaP{2}-complete problem of deciding if the
		maximum number of variables set to true in any satisfying assignment of a (satisfiable)
		propositional formula~$\psi$ is odd
		\cite{Krentel88,Wagner90}. 
	\end{proof}
	
	\duplicateprop{6}%
	{Given an agenda~$\Phi$ and a profile~$\Pf$,
		the problem of deciding whether there is some~$J^{*} \in \mathcal{J}(\Phi)$
		with~$\max_{J',J'' \in \Pf} |H(J^{*},J') - H(J^{*},J'')| = 0$
		is \NP{}-complete.
		Moreover, \NP{}-hardness holds even for the case where~$\Phi$
		consists of logically independent issues---i.e., the case
		where~$\mathcal{J}(\Phi) = \{0,1\}^m$ for some~$m$.}
	\begin{proof}[Proof (sketch)]
		The problem of deciding if there exists a truth assignment that satisfies a given
		3CNF formula~$\psi$ and that sets at least, say,~2 variables among~$\psi$ to true
		is an \NP{}-complete problem.
		Then, by combining Claim~\ref{claim:maxeq-hardness-1}
		in the proof of Theorem~\ref{thm:fpnplogwit-completeness}
		with Lemma~\ref{lem:lemma-maxeq-hardness},
		we directly get that the problem of deciding whether there is some~$J^{*} \in \mathcal{J}(\Phi)$
		with~$\max_{J',J'' \in \Pf} |H(J^{*},J') - H(J^{*},J'')| = 0$, for a given agenda~$\Phi$
		and a given profile~$\Pf$ is \NP{}-hard, even for the case where~$\Phi$
		consists of logically independent issues.
		Membership in \NP{} follows from the fact that one can guess a set~$J^{*}$
		(together with a truth assignment witnessing that it is consistent) in polynomial time,
		after which checking whether~$J^{*} \in \mathcal{J}(\Phi)$---i.e., whether it is indeed consistent---%
		and checking whether~$\max_{J',J'' \in \Pf} |H(J^{*},J') - H(J^{*},J'')| = 0$ can then be done
		in polynomial time.
	\end{proof}
	
	\DeclareRobustCommand{\DE}[3]{#3}
	
	\bibliographystyle{ACM-Reference-Format}  
	\bibliography{egal}

\end{document}